\documentclass[11pt,letterpaper]{article}
\usepackage[T1]{fontenc}
\usepackage[margin=1in]{geometry}
\usepackage{times}
\usepackage[authoryear,round]{natbib}
\setcitestyle{citesep={;},aysep={,},yysep={;}}
\usepackage{microtype}
\usepackage{etoolbox}
\usepackage{parskip}
\usepackage{titlesec}
\usepackage{needspace}
\pretocmd{\section}{\Needspace{6\baselineskip}}{}{}
\titlespacing*{\section}{0pt}{2ex plus 0.5ex minus 0.2ex}{1ex plus 0.2ex}
\titlespacing*{\subsection}{0pt}{1.8ex plus 0.5ex minus 0.2ex}{0.8ex plus 0.2ex}
\titlespacing*{\paragraph}{0pt}{1.5ex plus 0.5ex minus 0.2ex}{1em}
\usepackage{amsmath,amssymb,amsthm,mathtools}
\usepackage{graphicx,booktabs,array,multirow,longtable}
\usepackage{flafter,placeins}
\let\preprinttable\table
\renewcommand{\table}[1][htbp]{\preprinttable[!htbp]}
\BeforeBeginEnvironment{longtable}{\par\Needspace{0.85\textheight}}
\usepackage{algorithm,algpseudocode}
\AtBeginEnvironment{algorithmic}{\small}
\usepackage{xcolor}
\usepackage{xurl}
\usepackage{hyperref}
\hypersetup{
  colorlinks=true,
  linkcolor=blue!45!black,
  citecolor=blue!45!black,
  urlcolor=blue!45!black,
  pdfauthor={Yupeng Qiu; Han Fang; Ee-Chien Chang},
  pdftitle={ClockMark: Adaptive and Host-Dependent Watermarking for Large Language Models},
  pdfsubject={Large language model watermarking},
  pdfkeywords={large language models, multi-bit watermarking, distribution preservation, robustness}
}

\apptocmd{\thebibliography}{\interlinepenalty=10000}{}{}
\newtheorem{theorem}{Theorem}

\newcommand{\papertablestyle}{%
  \centering\footnotesize
  \renewcommand{\arraystretch}{1.05}%
  \setlength{\heavyrulewidth}{0.6pt}%
  \setlength{\lightrulewidth}{0.35pt}%
  \setlength{\cmidrulewidth}{0.3pt}%
}

\newcommand{\E}{\mathbb E}
\newcommand{\Prb}{\mathbb P}

\newcommand{\R}{\mathbb R}

\DeclareMathOperator*{\argmin}{arg\,min}
\DeclareMathOperator{\Exp}{Exp}
\DeclareMathOperator{\sgn}{sign}
\title{\LARGE\bfseries ClockMark: Adaptive and Host-Dependent Watermarking for Large Language Models}
\author{%
  Yupeng Qiu\textsuperscript{1}\qquad
  Han Fang\textsuperscript{2}\thanks{Corresponding author: \href{mailto:fanghan@ustc.edu.cn}{\texttt{fanghan@ustc.edu.cn}}.}\qquad
  Ee-Chien Chang\textsuperscript{1}\\[0.8em]
  \parbox{0.95\textwidth}{\centering\small
    \textsuperscript{1}School of Computing, National University of Singapore, Singapore\\[0.4em]
    \textsuperscript{2}School of Cyber Science and Technology,\\
    University of Science and Technology of China, Hefei, Anhui, China
  }
}
\date{}
\begin{document}
\maketitle
\begin{abstract}
Large language model (LLM) watermarking operates on the next-token distribution that evolves during generation. In streaming generation, each sampled token becomes part of the text and affects subsequent distributions. Embedding therefore proceeds with a changing host, unlike classical watermarking, which adapts embedding to an existing host signal. To accommodate this difference, we formulate a novel embedding framework for LLM watermarking based on two principles. We first argue that embedding should be adaptive, retaining the state of previous sampling decisions to better guide subsequent choices. We further argue that embedding should be host-dependent, using the current LLM distribution to constrain both token selection and state updates. Guided by these principles and the two requirements of distribution preservation and robustness to text editing, we propose ClockMark. It combines adaptive and host-dependent embedding with message spreading, embedding the full payload in each watermarked token. Theoretically, we prove that ClockMark preserves the original LLM sequence distribution, thereby maintaining generation quality. Empirically, extensive experiments against 12 baselines demonstrate improved robustness to text editing, with gains of up to 16.08 percentage points in bit accuracy over the strongest tested baseline.
\end{abstract}
\section{Introduction}
\label{sec:intro}
Texts generated by large language models (LLMs) are increasingly difficult to distinguish from human-written content \citep{jakesch2023human,porter2024poetry}. Text watermarking offers a practical way to verify model-generated text by embedding secret signals during generation. Earlier zero-bit methods can only determine whether the given text is watermarked, which is insufficient for more detailed attribution \citep{kirchenbauer2023watermark,hu2024unbiased,kuditipudi2024distortionfree,dathathri2024scalable}. Multi-bit watermarking addresses this limitation by embedding identifiers, such as model or user IDs, that can later be recovered for tracing \citep{wang2024codable,yoo2024mpac}.

However, achieving robustness to text editing while preserving generation quality remains challenging. Existing methods improve robustness through balanced allocation and structured coding \citep{yoo2024mpac,qu2025robust,xu2026xmark}. More recent methods broaden watermark signal coverage by allowing each token to carry signals for all message bits \citep{gilani2026arcmark,gloaguen2026everybit}. In parallel, distribution-preserving approaches use unbiased reweighting or keyed sampling to maintain generation quality \citep{feng2025bimark,jiang2025stealthink,cui2026mc2mark,wang2026pura}. For methods that derive watermark randomness from preceding text, repeated contexts introduce a further difficulty. Reusing the same randomness at these contexts can bias generation, whereas skipping them sacrifices embedding opportunities \citep{dathathri2024scalable,feng2025bimark,jiang2026mirrormark,wang2026pura}. Increasing the watermark context length reduces repetition but allows a single edit to disrupt watermark signals in more subsequent tokens.

Addressing these challenges requires considering how watermark embedding interacts with the LLM generation process. LLM watermarking operates on the next-token distribution, which becomes available one step at a time and depends on the previously generated text. In streaming generation, each sampled token becomes part of the text and affects the distributions used for subsequent sampling. Embedding therefore proceeds through a sequence of sampling choices under a changing host, unlike classical watermarking, which adapts embedding to an existing host signal \citep{moulin2003information}.

To accommodate this difference, we formulate a novel embedding framework for LLM watermarking based on two principles. Adaptive embedding retains the state of previous sampling decisions and uses it to guide subsequent choices. Host-dependent embedding uses the current LLM distribution to constrain both token selection and state updates. A concrete method developed within this framework should also meet two requirements: preserving the LLM sequence distribution and supporting reliable message recovery after text editing.

Following this framework, we propose ClockMark. The embedding process is a race among candidate tokens. Each candidate starts with a countdown determined by the secret key and all payload bits, so each sampling decision depends jointly on the full payload without position allocation. The current LLM probability determines how quickly each countdown decreases, and the first candidate to reach zero is selected as the next token. To implement adaptive embedding, ClockMark gives only the selected candidate a new countdown and retains the unfinished countdowns of the others. It saves these values for each context and resumes them when that context reappears. This allows ClockMark to use one-token contexts without skipping repetitions, limiting the number of subsequent watermark signals affected by an edit. To implement host-dependent embedding, all candidates advance for the same duration at rates set by their current LLM probabilities. These probabilities therefore govern both token selection and the unfinished countdowns retained for future decisions. Theoretically, we prove that ClockMark preserves the original LLM sequence distribution, thereby maintaining generation quality. Empirically, extensive experiments against 12 baselines demonstrate improved robustness to text editing, with gains of up to 16.08 percentage points in bit accuracy over the strongest tested baseline.

Our main contributions are as follows:
\begin{itemize}
    \item We formulate an adaptive and host-dependent framework for LLM watermarking that models embedding as a sequential process guided by the evolving host distribution and the state retained from previous sampling decisions.
    \item We propose ClockMark, which combines full-payload embedding with adaptive and host-dependent sampling. We theoretically prove that the resulting process preserves the original LLM sequence distribution, thereby maintaining generation quality.
    \item We develop a unified receiver that supports message recovery and watermark detection through a shared statistic computed from the received text and secret key. We establish theoretical false-positive control for detection.
    \item Extensive experiments against 12 multi-bit watermarking baselines demonstrate improved robustness to text editing while maintaining generation quality, with gains of up to 16.08 percentage points in bit accuracy over the strongest tested baseline.
\end{itemize}

\section{Related Work}
\label{sec:related}

\paragraph{Robust Embedding.}
Robust embedding depends on how watermark signals are distributed across tokens and how they depend on the preceding context. Early multi-bit methods encode messages through message-dependent hashing or shifts of pseudorandom scores \citep{wang2024codable,fernandez2023three}. Position-allocation methods assign tokens to individual message positions or segments \citep{yoo2024mpac}. Structured coding and balanced allocation reduce vulnerability to text editing by addressing uneven message coverage \citep{qu2025robust,xu2026xmark,jiang2026mirrormark}. More recent methods combine signals for multiple bits within each token through layered embedding \citep{cui2026mc2mark,park2026weavemark} or constraints on the XOR of several bits \citep{agrawal2026mosaic}. ArcMark and EBEW allow each token to carry signals for all message bits \citep{gilani2026arcmark,gloaguen2026everybit}. Allocation state and feedback have also been used to improve message coverage during generation \citep{jiang2026mirrormark,gloaguen2026everybit}. These designs distribute message information more broadly, reducing dependence on separate token subsets.

Context dependence introduces a separate source of vulnerability. An edit can disrupt watermark signals in subsequent unchanged tokens, even when signals are distributed evenly across message bits \citep{fernandez2023three,yoo2024mpac}. Shorter contexts limit this damage but repeat more frequently. Skipping repeated contexts consequently leaves fewer positions available for embedding \citep{feng2025bimark,jiang2026mirrormark,wang2026pura}. ClockMark addresses both issues by combining message spreading with adaptive sampling: each watermarked token carries information about all message bits, while retained sampling progress allows repeated one-token contexts to be reused for embedding and limits the spread of edit damage.

\paragraph{Distribution-Preserving Embedding.}
Preserving text quality requires both a distribution-preserving sampling rule and a way to maintain this property when contexts repeat. Biased embedding strengthens watermark signals by changing token probabilities \citep{kirchenbauer2023watermark}. Probability-aware constraints can reduce the resulting quality loss, but do not generally preserve the original distribution \citep{wang2024codable,park2026watermod,jiang2024credid}. Distribution-preserving methods follow two main approaches: unbiased reweighting modifies token probabilities while recovering the original distribution in expectation \citep{hu2024unbiased,wu2024dipmark,chen2025mcmark}, whereas keyed sampling embeds watermark signals through the randomness used to select tokens \citep{aaronson2023watermarking,kuditipudi2024distortionfree,dathathri2024scalable}. Multi-bit methods extend these approaches by making either the reweighting or the sampling rule depend on the message \citep{feng2025bimark,jiang2025stealthink,cui2026mc2mark,park2026weavemark,gilani2026arcmark,zhu2026quantilemark,wang2026pura}.

Context dependence raises a separate issue: preserving the distribution of an individual sampling step does not automatically preserve it when the same randomness is reused. Reusing the same sampling randomness at repeated contexts can introduce bias into generation and produce repetitive text, degrading text quality \citep{dathathri2024scalable}. A common solution is to skip watermark embedding at repeated contexts and sample from the original LLM distribution instead \citep{hu2024unbiased,dathathri2024scalable,feng2025bimark,wang2026pura}. Increasing the context length also reduces the probability of repetition \citep{dathathri2024scalable}. These strategies help preserve text quality, but skipping repetitions sacrifices embedding opportunities, while longer contexts expose more watermark signals to each edit, potentially reducing robustness. ClockMark instead retains sampling progress and updates it according to the current LLM probabilities, allowing continued embedding at repeated contexts while preserving the output distribution.

\begin{figure}[t]
\centering
\includegraphics[width=\linewidth]{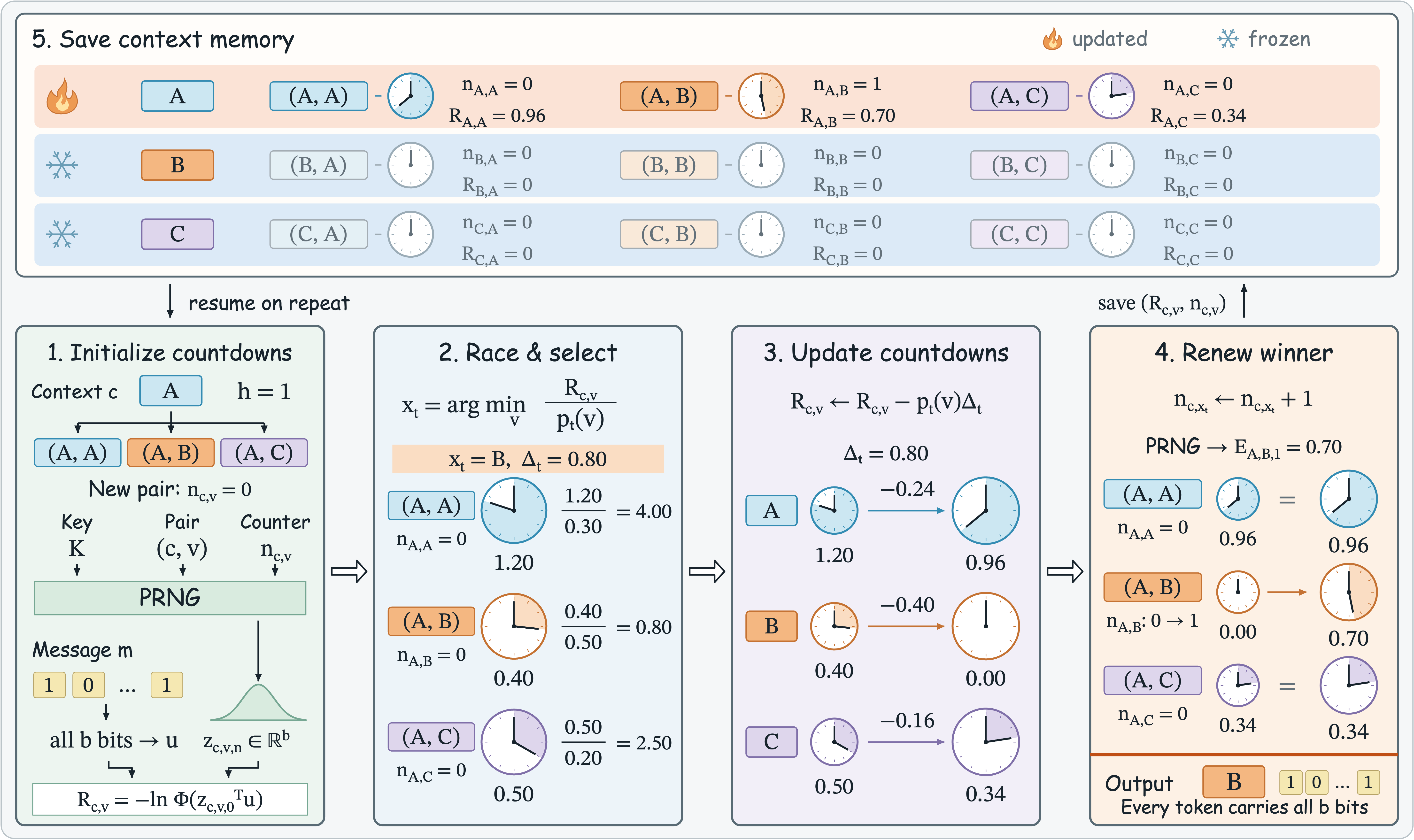}
\caption{ClockMark embedding. Each candidate countdown depends on all message bits and advances at its LLM probability. The sampler selects the first to finish, updates all active countdowns, and renews only the winner. Remaining countdowns are saved for repeated contexts. The indices $c$ and $v$ denote the context and candidate, and $n_{c,v}$ counts their pair selections.}
\label{fig:embedding}
\end{figure}

\section{Method}
\label{sec:method}\label{sec:framework}
ClockMark implements adaptive and host-dependent embedding through a sampling process that retains its state across repeated contexts. Message spreading makes each candidate token depend on all message bits. Adaptive sampling retains the progress of earlier sampling decisions, while host-dependent sampling uses the current LLM probabilities to select tokens and update that progress. Figure~\ref{fig:embedding} illustrates how these designs work together. The receiver reconstructs watermark signals from adjacent token pairs and sums them for both message recovery and watermark detection.

\subsection{Notation}
\label{sec:notation}
Let $\mathcal V$ be the vocabulary and $x_{1:T}=(x_1,\ldots,x_T)$ a generated sequence of $T$ tokens. At step $t$, the LLM $M$ assigns candidate token $v\in\mathcal V$ probability $P_M(v\mid x_{<t})$, where $x_{<t}=(x_1,\ldots,x_{t-1})$ is the preceding text. We omit the dependence on the fixed input prompt and write $p_t(v)=P_M(v\mid x_{<t})$ for brevity. The message is $m\in\{0,1\}^b$, where $b$ is the number of bits. The shared secret key $K\in\mathcal K$ is sampled from distribution $P_K$ over the key space $\mathcal K$. We write $P_{M,w}(x_{1:T}\mid m,K)$ for the probability of generating $x_{1:T}$ with the watermark. Given the received text $\widetilde x$ and $K$, the receiver recovers an estimate $\widehat m$ and detects the watermark. The sender and receiver use the same tokenizer and message length.

\subsection{Watermark Embedding}
\label{sec:intervals}\label{sec:sender}

\paragraph{Message Spreading.}
To let each watermarked token carry information about all message bits, we combine all bits into a score for each candidate token. We represent the message by a signed vector and normalize it:
\begin{equation}
 y=2m-\mathbf1,\qquad u(y)=\frac{y}{\sqrt b}.
 \label{eq:direction}
\end{equation}
Here $\mathbf1$ is the all-ones vector, so bit 0 maps to $-1$ and bit 1 to $+1$. The unit vector $u(y)$ specifies the message direction. We write $u=u(y)$ when the message is fixed.

At step $t\ge2$, let $c=x_{t-1}$ be the one-token watermark context and $v\in\mathcal V$ a candidate for the next token. For each pair $(c,v)$, the counter $n_{c,v}$ records how many times this pair has been selected and is initialized to zero. The key $K$, pair $(c,v)$, and current counter $n_{c,v}$ jointly determine the reproducible pseudorandom vector $z_{c,v,n_{c,v}}$. For any fixed counter value $n$, we model $z_{c,v,n}\sim\mathcal N(0,I_b)$, where $I_b$ is the $b\times b$ identity matrix. Vectors indexed by distinct $(c,v,n)$ are modeled as independent.

Each candidate token $v$ has a corresponding vector $z_{c,v,n}$. Its inner product $z_{c,v,n}^{\top}u$ measures alignment with the message direction $u$ and serves as the candidate score. Since $u$ represents all message bits, each candidate score depends on the entire message without position allocation.

\paragraph{Adaptive and Host-Dependent Sampling.}
We use these scores to construct a race among candidate tokens, where the first candidate to finish is selected. Each score determines a positive countdown interval:
\begin{equation}
 E_{c,v,n}(y)=-\ln\Phi\!\left(z_{c,v,n}^{\top}u(y)\right),
 \label{eq:gap}
\end{equation}
where $\Phi$ is the standard normal cumulative distribution function. A larger score produces a shorter countdown. We abbreviate $E_{c,v,n}(y)$ as $E_{c,v,n}$ when the message is fixed.

To implement adaptive embedding, ClockMark stores $(R_{c,v},n_{c,v})$ in a context memory, where $R_{c,v}$ is the remaining amount of the current countdown. Memory starts empty for each generated sequence. For the current context $c$, a new candidate with positive probability starts with $R_{c,v}=E_{c,v,0}$ and counter zero. Previously initialized candidates resume their saved countdowns. Thus earlier sampling progress is retained and used when a context reappears.

To implement host-dependent embedding, we use the current LLM probability $p_t(v)$ as the rate at which candidate $v$ decreases its remaining countdown. Its completion time is $R_{c,v}/p_t(v)$, so the selected token and elapsed race time are
\begin{equation}
 x_t=\argmin_{v:p_t(v)>0}\frac{R_{c,v}}{p_t(v)},
 \qquad
 \Delta_t=\frac{R_{c,x_t}}{p_t(x_t)}.
 \label{eq:race}
\end{equation}
This selection rule combines adaptive and host-dependent embedding through the saved countdown $R_{c,v}$ and the current probability $p_t(v)$.

All candidates in the race advance for the same duration $\Delta_t$. Since candidate $v$ advances at rate $p_t(v)$, its remaining countdown decreases by $p_t(v)\Delta_t$:
\begin{equation}
 R_{c,v}\leftarrow R_{c,v}-p_t(v)\Delta_t
 \quad\text{for all }v\text{ with }p_t(v)>0.
 \label{eq:advance}
\end{equation}
The selected countdown has now reached zero. Only the selected pair receives a fresh interval:
\begin{equation}
 n_{c,x_t}\leftarrow n_{c,x_t}+1,\qquad
 R_{c,x_t}\leftarrow E_{c,x_t,n_{c,x_t}}.
 \label{eq:renew}
\end{equation}
The incremented counter selects the next interval from Equation~\eqref{eq:gap}. Unselected candidates keep their counters and updated countdowns. Pairs belonging to other contexts and candidates with zero probability remain unchanged throughout the race.

We save the updated state for context $c$ and continue with $x_t$ as the next context. When $c$ reappears, its countdowns resume at rates set by the current LLM probabilities. The model probabilities therefore govern both token selection and the progress retained in memory. This allows ClockMark to use one-token contexts without skipping repetitions, limiting the number of subsequent watermark signals affected by an edit. These operations follow the initialization, race, countdown update, renewal, and memory storage shown in Figure~\ref{fig:embedding}. Algorithm~\ref{alg:sender} in Appendix~\ref{app:algorithm} gives the full procedure.

\begin{theorem}[Distribution preservation]
\label{thm:preservation}
For any fixed prompt and message $m\in\{0,1\}^b$, the probability of every token sequence $x_{1:T}\in\mathcal V^T$ is preserved in expectation over the secret key:
\begin{equation}
 \E_{K\sim P_K}\!\left[P_{M,w}(x_{1:T}\mid m,K)\right]
 =P_M(x_{1:T})=\prod_{t=1}^T p_t(x_t).
 \label{eq:preserve}
\end{equation}
\end{theorem}
Appendix~\ref{app:theory} gives the proof idea (Section~\ref{app:proof-idea}), the single-race calculation (Section~\ref{app:race-invariance}), and the sequence-level proof (Section~\ref{app:preservation}).

\subsection{Watermark Recovery}
\label{sec:evidence}\label{sec:recovery}
The receiver uses the shared key $K$ and received text $\widetilde x$ to reconstruct watermark signals without access to the LLM or the context memory used during embedding. As shown in Figure~\ref{fig:decoding}, it extracts distinct adjacent token pairs, reconstructs their initial vectors, and recovers the message from the signs of their sum.

\begin{figure}[t]
\centering
\includegraphics[width=\linewidth]{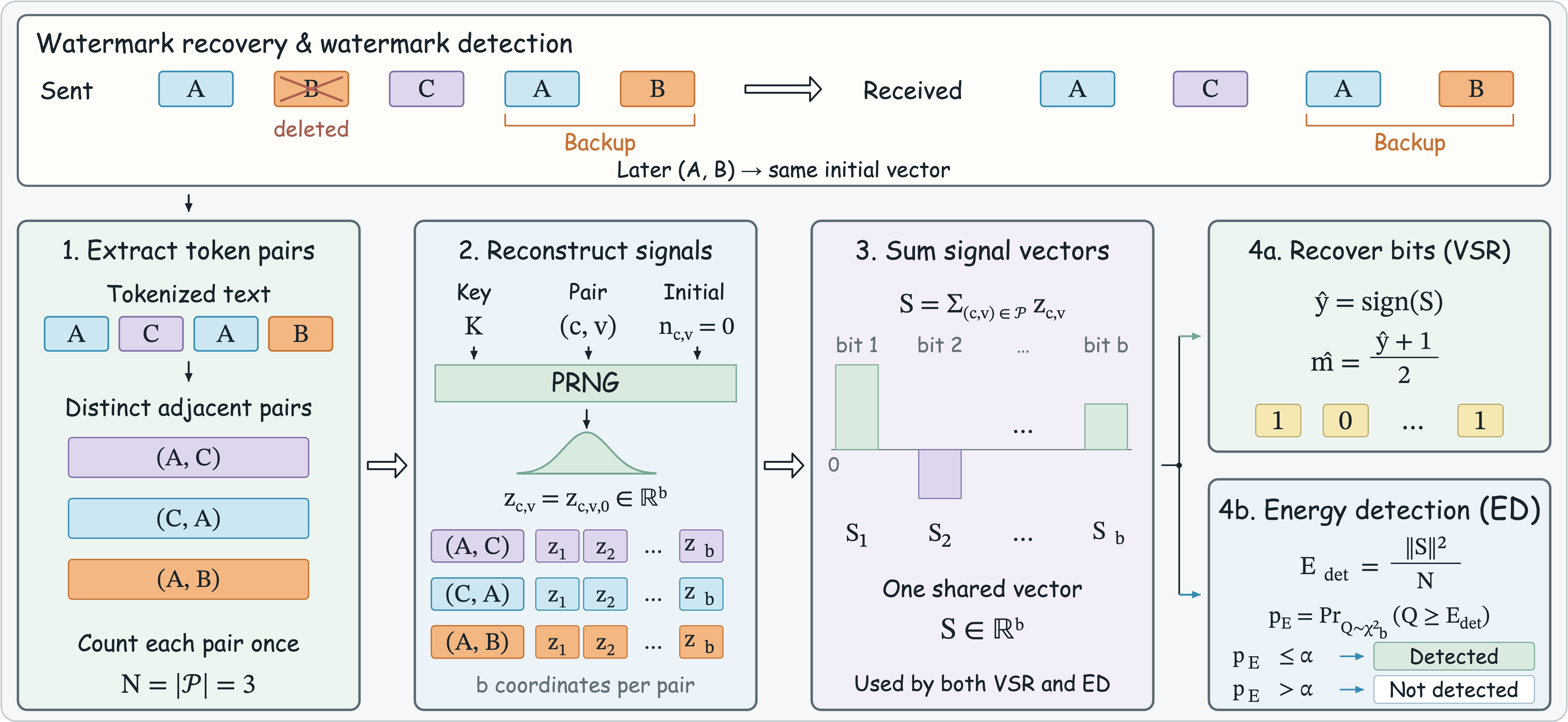}
\caption{ClockMark recovery and detection. The key reconstructs one initial vector for each of $N$ distinct adjacent token pairs. Their sum $S$ yields the message through coordinate signs (VSR) and a detection statistic through normalized energy $\|S\|^2/N$ (ED). A surviving $(A,B)$ preserves access to the same initial vector after deletion. Each pair is counted once.}
\label{fig:decoding}
\end{figure}

\paragraph{Token Pair Extraction.}
With the one-token context used during embedding, each watermark signal is associated with an adjacent token pair $(c,v)$. For received text of $\widetilde T$ tokens, the receiver collects
\begin{equation}
 \mathcal P(\widetilde x)=\{(\widetilde x_{t-1},\widetilde x_t):2\le t\le\widetilde T\},
 \qquad N=|\mathcal P(\widetilde x)|,
 \label{eq:evidence}
\end{equation}
where $N$ is the number of distinct pairs. We write $\mathcal P=\mathcal P(\widetilde x)$ below. The set also includes any pairs created by editing.

\paragraph{Watermark Signal Reconstruction.}
For each $(c,v)\in\mathcal P$, the receiver regenerates the initial vector $z_{c,v}:=z_{c,v,0}\in\R^b$ using the same pseudorandom generator as embedding, with inputs $K$, $(c,v)$, and counter value zero.

A pair first appears in the original generated text when its initial countdown completes. The receiver can retrieve that initial vector from any surviving occurrence of the pair, even if editing removes its first occurrence. Repeated pairs therefore provide backups for reconstructing watermark signals, as illustrated by ``$(A,B)$'' in Figure~\ref{fig:decoding}. Each pair contributes only once because its repeated occurrences reconstruct the same vector. Fixing the counter at zero also removes the need to reproduce the sampling history.

\paragraph{Vector-Sum Recovery.}
During embedding, a larger projection $z_{c,v,0}^{\top}u$ gives a shorter initial countdown and makes the pair more likely to occur. This favors vectors aligned with the message direction.

This motivates Vector-Sum Recovery (VSR):
\begin{equation}
 S=\sum_{(c,v)\in\mathcal P}z_{c,v},\qquad
 \widehat y=\sgn(S),\qquad \widehat m=(\widehat y+\mathbf1)/2.
 \label{eq:sum-recovery}
\end{equation}
Here $S\in\R^b$ is the vector sum and $\widehat y$ is the recovered signed message. The sign function acts coordinatewise, with $\sgn(0)=+1$, so a nonnegative coordinate yields bit 1 and a negative coordinate bit 0. Because message spreading associates every vector with all message bits, each pair contributes to all $b$ coordinates of $S$. The same sum $S$ supports watermark detection in Section~\ref{sec:detection}.

\subsection{Watermark Detection}
\label{sec:detection}
VSR returns a bit string even for unwatermarked text, so message recovery alone does not establish watermark presence. Energy Detection (ED) tests for a watermark using the same vector sum $S$, as shown in Figure~\ref{fig:decoding}. It uses the watermark signals from message spreading without knowing the original message or allocating a separate detection signal.

\paragraph{Detection Statistic.}
Each coordinate of $S$ aggregates the watermark signal for one message bit. Its sign determines the recovered bit, while its magnitude indicates the accumulated signal strength. ED combines the squared coordinates and normalizes by the number of distinct pairs:
\begin{equation}
 E_{\rm det}=\frac{\|S\|^2}{N}
 =\sum_{j=1}^{b}\left(\frac{S_j}{\sqrt N}\right)^2,
 \label{eq:energy}
\end{equation}
where $S_j$ is the $j$th coordinate of $S$, $\|S\|^2=\sum_{j=1}^{b}S_j^2$ is its squared Euclidean norm, and $N>0$ is the number of distinct pairs.

Under the unwatermarked null hypothesis $H_0$, the received text is independent of the secret key. We fix the received text and its distinct token pairs, and consider randomness over $K\sim P_K$. Under the Gaussian construction in Section~\ref{sec:intervals}, their reconstructed vector coordinates are independent standard normal variables. Each $S_j$ is therefore the sum of $N$ such variables, with mean zero and variance $N$. Dividing by $\sqrt N$ gives
\begin{equation}
 \frac{S_j}{\sqrt N}\sim\mathcal N(0,1),
 \qquad j=1,\ldots,b,
 \label{eq:energy-standardization}
\end{equation}
with independence across coordinates. Equation~\eqref{eq:energy} is thus the sum of squares of $b$ independent standard normal variables, which defines the chi-square distribution with $b$ degrees of freedom. Hence $E_{\rm det}\sim\chi_b^2$ under $H_0$, with the same null distribution for any $N>0$.

\paragraph{Detection Rule.}
We compute the upper-tail p-value
\begin{equation}
 p_E=\Prb_{Q\sim\chi_b^2}(Q\ge E_{\rm det}),
 \label{eq:detection-pvalue}
\end{equation}
where $Q$ follows the null distribution. This is the probability, under $H_0$, of obtaining an energy at least as large as the observed value. At significance level $\alpha\in(0,1)$, we declare the watermark present when $p_E\le\alpha$. Under $H_0$, this rule has a false-positive probability of at most $\alpha$, as proved in Appendix~\ref{app:energy}.

\section{Experiments}
\label{sec:experiments}

\subsection{Experimental Setup}
\paragraph{Models and Datasets}
We use Qwen2.5-7B-Instruct \citep{qwen2024qwen25} and Gemma-2-9B \citep{gemmateam2024gemma2improvingopen} for generation and evaluate on LFQA \citep{fan-etal-2019-eli5} and C4-RealNewsLike \citep{JMLR:v21:20-074}. We implement the experiments in PyTorch \citep{NEURIPS2019_bdbca288} and run them on four NVIDIA A40 GPUs.\\
\textbf{Baselines and Configuration}
We compare ClockMark with twelve multi-bit watermarks: BalanceMarking \citep{wang2024codable}, BiMark \citep{feng2025bimark}, CredID \citep{jiang2024credid}, MOSAIC \citep{agrawal2026mosaic}, MPAC \citep{yoo2024mpac}, PURA \citep{wang2026pura}, QuantileMark \citep{zhu2026quantilemark}, RSBH \citep{qu2025robust}, StealthInk \citep{jiang2025stealthink}, WaterMod \citep{park2026watermod}, WeaveMark \citep{park2026weavemark}, and XMark \citep{xu2026xmark}. Preserving text quality is essential for commercial LLM deployment, where watermarking should not compromise model performance or user experience. Gemini and recent Claude models also use distribution-preserving watermarking \citep{dathathri2024scalable}.\footnote{Claude deployment: \url{https://www.anthropic.com/news/claude-text-watermark}.} For a fair comparison, we select the strength parameters of biased baselines with the aim of keeping text quality at a comparable level. Appendix~\ref{app:protocol} lists the configurations in Appendix~\ref{app:baselines}. We evaluate robustness using three attacks. Deletion uniformly removes 10\%, 30\%, or 50\% of the generated tokens. Substitution follows the contextual word-replacement procedure of \citet{park2026weavemark}, using BERT suggestions and WordNet synonyms at the same rates. Paraphrasing uses DIPPER-XXL \citep{krishna2023dipper} with lexical and order diversity both set to 20.\\
\textbf{Evaluation Metrics}
Recovery is measured by bit accuracy (BA), the fraction of correctly recovered payload bits, with unrecovered bits counted as errors. Detection is measured by area under the ROC curve (AUC), which summarizes separation of watermarked and unwatermarked text across thresholds, and true positive rate (TPR) at a 1\% false positive rate (FPR), the fraction of watermarked texts correctly detected. For text quality, we report mean perplexity (PPL) under Pythia-1.4B-deduped \citep{pmlr-v202-biderman23a}, with lower values indicating a closer fit to its language distribution. 

\subsection{Text Quality, Recovery, and Detection}
Table~\ref{tab:clean} reports the Qwen comparison on C4. Appendix~\ref{app:full-results} provides Qwen results on LFQA (Appendix~\ref{app:clean-lfqa}) and Gemma results on both C4 and LFQA (Appendix~\ref{app:gemma}). For each model, dataset, and payload length, we generate 500 texts of 200 tokens per method.

\begin{table}[t]
\papertablestyle
\setlength{\tabcolsep}{4pt}
\caption{Clean Qwen results on C4. Bit accuracy (BA), area under the ROC curve (AUC), and true positive rate (TPR) at a 1\% false positive rate (FPR) are percentages. Bold marks the best displayed BA, AUC, and TPR. Baselines without reported detection appear first. Dashes indicate unavailable or inapplicable results.}\label{tab:clean}
\begin{tabular*}{\textwidth}{@{\extracolsep{\fill}}lrrrrrrrr@{}}
\toprule
\multicolumn{1}{c}{\multirow{2}{*}{Method}} & \multicolumn{4}{c}{\textbf{C4 (32 bits)}} & \multicolumn{4}{c}{\textbf{C4 (64 bits)}} \\
\cmidrule(lr){2-5}\cmidrule(lr){6-9}
 & PPL $\downarrow$ & BA $\uparrow$ & AUC $\uparrow$ & TPR $\uparrow$ & PPL $\downarrow$ & BA $\uparrow$ & AUC $\uparrow$ & TPR $\uparrow$ \\
\midrule
No watermark & 14.17 & -- & -- & -- & 14.17 & -- & -- & -- \\
\midrule
BalanceMarking & 14.35 & 78.74 & -- & -- & 14.11 & 62.34 & -- & -- \\
CredID & 15.76 & 79.06 & -- & -- & 15.38 & 64.49 & -- & -- \\
PURA & 14.46 & 98.00 & -- & -- & 14.15 & 93.07 & -- & -- \\
RSBH & 15.47 & 56.16 & -- & -- & 14.82 & 51.68 & -- & -- \\
WaterMod & 22.60 & 87.84 & -- & -- & 21.24 & 81.72 & -- & -- \\
\midrule
BiMark & 14.22 & 94.52 & 99.97 & 99.40 & 14.31 & 82.79 & 99.83 & 99.40 \\
MOSAIC & 16.50 & 58.59 & 54.92 & 10.40 & 16.13 & 51.83 & 49.97 & 0.00 \\
MPAC & 18.31 & 88.92 & 99.03 & 94.60 & 18.95 & 79.95 & 98.78 & 92.80 \\
QuantileMark & 14.09 & 93.11 & 99.02 & 95.60 & 13.95 & 84.62 & 99.23 & 89.40 \\
StealthInk & 14.41 & 86.78 & 98.51 & 88.60 & 14.07 & 78.46 & 96.29 & 74.60 \\
WeaveMark & 14.39 & 96.10 & 99.91 & 97.80 & 14.39 & 88.33 & 99.61 & 98.40 \\
XMark & 17.37 & 88.96 & 98.55 & 83.00 & 17.72 & 78.64 & 96.55 & 39.60 \\
\midrule
\textcolor{blue!45!black}{ClockMark} & 14.20 & \textbf{98.90} & \textbf{100.00} & \textbf{99.80} & 14.65 & \textbf{96.35} & \textbf{99.99} & \textbf{99.60} \\
\bottomrule
\end{tabular*}
\end{table}

\paragraph{Larger Payloads Widen the Recovery Lead}
ClockMark achieves the highest BA at both C4 payload lengths. Its lead over PURA, the strongest recovery baseline, grows from 0.90 to 3.29 percentage points as the payload increases from 32 to 64 bits, with 96.35\% BA at 64 bits. The LFQA lead similarly widens from 8.37 to 12.64 points (Table~\ref{tab:clean-lfqa}). PURA requires token probabilities from a source or proxy LLM for decoding \citep{wang2026pura}. ClockMark achieves these gains with model-agnostic recovery using only the received text and key.

\paragraph{Detection Remains Strong at Both Payload Lengths}
ClockMark achieves the highest AUC and TPR among methods with reported detection results on C4. TPR remains 99.80\% and 99.60\% at 32 and 64 bits, exceeding BiMark at 99.40\% for both. Detection therefore remains strong as the payload doubles. Although clean-text detection is near saturation for both methods, ClockMark combines this performance with higher recovery accuracy.

\paragraph{Recovery Gains Accompany Near-Baseline Perplexity}
ClockMark has PPL of 14.20 and 14.65 on C4 at 32 and 64 bits, respectively, compared with 14.17 for unwatermarked text. PURA and BiMark also remain close to the unwatermarked reference, so the recovery advantage extends to baselines with similar evaluator perplexity. On LFQA, PPL remains close to the unwatermarked reference of 11.05, at 10.97 and 11.07 for the two payload lengths (Table~\ref{tab:clean-lfqa}).

\paragraph{State-of-the-Art Performance on Gemma}
ClockMark achieves the highest BA in all four Gemma settings, reaching 99.57\% on LFQA and 99.52\% on C4 at 32 bits, and 97.33\% and 97.02\%, respectively, at 64 bits (Table~\ref{tab:gemma}). Detection remains near perfect at both payload lengths, with AUC above 99.99\% and TPR of 99.80\% on LFQA and 100.00\% on C4, matching the best reported TPR. PPL on LFQA and C4 is 13.22 and 12.78 at 32 bits and 13.20 and 12.90 at 64 bits, compared with unwatermarked values of 13.42 and 12.47. These results extend leading recovery and detection performance to another model family while maintaining generation quality.

\subsection{Robustness to Editing}
\label{sec:robustness}
We apply deletion, substitution, and paraphrasing to 200 token Qwen outputs with 32 bit payloads. Each dataset contains 500 matched texts. Table~\ref{tab:robust} reports C4, and Appendix~\ref{app:robust-lfqa} gives LFQA results. ClockMark achieves the highest BA and TPR in all seven attack conditions on both datasets. 
\begin{table}[t]
\papertablestyle
\setlength{\tabcolsep}{2pt}
\caption{C4 robustness for 32 bit Qwen watermarks. Bit accuracy (BA), area under the ROC curve (AUC), and true positive rate (TPR) at a 1\% false positive rate (FPR) are percentages. Bold marks the best displayed BA, AUC, and TPR. Baselines without reported detection appear first. Dashes indicate unavailable or inapplicable results.}\label{tab:robust}
\begin{tabular*}{\textwidth}{@{\extracolsep{\fill}}lrrrrrrrrr@{}}
\toprule
\multicolumn{1}{c}{\multirow{3}{*}[-4pt]{Method}} & \multicolumn{9}{c}{\textbf{Deletion}} \\
\cmidrule(lr){2-10}
 & \multicolumn{3}{c}{10\%} & \multicolumn{3}{c}{30\%} & \multicolumn{3}{c}{50\%} \\
\cmidrule(lr){2-4}\cmidrule(lr){5-7}\cmidrule(lr){8-10}
 & BA & AUC & TPR & BA & AUC & TPR & BA & AUC & TPR \\
\midrule
BalanceMarking & 58.69 & -- & -- & 38.40 & -- & -- & 24.74 & -- & -- \\
CredID & 60.93 & -- & -- & 38.20 & -- & -- & 25.27 & -- & -- \\
PURA & 90.17 & -- & -- & 69.14 & -- & -- & 56.56 & -- & -- \\
RSBH & 53.60 & -- & -- & 51.31 & -- & -- & 50.87 & -- & -- \\
WaterMod & 62.63 & -- & -- & 55.95 & -- & -- & 52.35 & -- & -- \\
\midrule
BiMark & 88.39 & 99.77 & 98.44 & 70.79 & 95.94 & 74.44 & 52.65 & 72.35 & 11.78 \\
MOSAIC & 52.56 & 50.04 & 0.67 & 51.06 & 49.82 & 0.22 & 50.94 & 49.93 & 0.44 \\
MPAC & 85.85 & 98.43 & 92.22 & 77.42 & 94.28 & 61.78 & 67.03 & 80.43 & 18.44 \\
QuantileMark & 77.06 & 91.38 & 49.78 & 56.97 & 55.11 & 2.00 & 52.19 & 52.02 & 0.67 \\
StealthInk & 78.50 & 89.99 & 40.22 & 61.78 & 58.29 & 2.44 & 53.37 & 47.07 & 0.44 \\
WeaveMark & 87.26 & 99.55 & 93.56 & 57.58 & 93.81 & 51.78 & 50.76 & 76.14 & 8.89 \\
XMark & 80.85 & 85.64 & 10.00 & 65.86 & 18.21 & 0.00 & 56.71 & 0.62 & 0.00 \\
\midrule
\textcolor{blue!45!black}{ClockMark} & \textbf{97.89} & \textbf{99.95} & \textbf{99.33} & \textbf{92.58} & \textbf{99.12} & \textbf{96.22} & \textbf{81.03} & \textbf{93.83} & \textbf{59.56} \\
\bottomrule
\end{tabular*}
\par\vspace{6pt}
\begin{tabular*}{\textwidth}{@{\extracolsep{\fill}}lrrrrrrrrrrrr@{}}
\toprule
\multicolumn{1}{c}{\multirow{3}{*}[-4pt]{Method}} & \multicolumn{9}{c}{\textbf{Substitution}} & \multicolumn{3}{c}{\multirow{2}{*}[-2pt]{\textbf{Paraphrasing}}} \\
\cmidrule(lr){2-10}
 & \multicolumn{3}{c}{10\%} & \multicolumn{3}{c}{30\%} & \multicolumn{3}{c}{50\%} & \multicolumn{3}{c}{} \\
\cmidrule(lr){2-4}\cmidrule(lr){5-7}\cmidrule(lr){8-10}\cmidrule(lr){11-13}
 & BA & AUC & TPR & BA & AUC & TPR & BA & AUC & TPR & BA & AUC & TPR \\
\midrule
BalanceMarking & 57.73 & -- & -- & 51.68 & -- & -- & 49.74 & -- & -- & 49.68 & -- & -- \\
CredID & 61.96 & -- & -- & 52.35 & -- & -- & 50.35 & -- & -- & 50.58 & -- & -- \\
PURA & 81.45 & -- & -- & 62.11 & -- & -- & 54.92 & -- & -- & 69.15 & -- & -- \\
RSBH & 52.14 & -- & -- & 50.97 & -- & -- & 50.31 & -- & -- & 51.05 & -- & -- \\
WaterMod & 58.43 & -- & -- & 54.16 & -- & -- & 52.40 & -- & -- & 53.17 & -- & -- \\
\midrule
BiMark & 82.36 & 98.97 & 92.89 & 63.12 & 83.05 & 22.67 & 52.26 & 64.40 & 6.22 & 67.42 & 89.18 & 48.00 \\
MOSAIC & 51.72 & 50.04 & 0.67 & 50.15 & 50.15 & 0.89 & 50.41 & 50.04 & 0.67 & 50.94 & 49.93 & 0.44 \\
MPAC & 80.62 & 96.12 & 74.44 & 67.55 & 78.56 & 16.22 & 59.47 & 56.93 & 3.11 & 68.94 & 77.92 & 22.89 \\
QuantileMark & 69.43 & 76.30 & 16.22 & 55.12 & 51.08 & 1.11 & 51.53 & 50.63 & 1.78 & 59.99 & 62.15 & 3.33 \\
StealthInk & 70.94 & 80.54 & 15.11 & 57.44 & 56.10 & 2.22 & 52.93 & 49.09 & 0.67 & 60.68 & 60.49 & 1.78 \\
WeaveMark & 73.12 & 98.39 & 82.00 & 51.58 & 83.21 & 24.00 & 50.06 & 66.66 & 6.44 & 53.90 & 87.47 & 37.11 \\
XMark & 74.50 & 83.15 & 8.67 & 60.86 & 55.82 & 0.44 & 53.99 & 49.29 & 1.11 & 61.50 & 50.01 & 1.78 \\
\midrule
\textcolor{blue!45!black}{ClockMark} & \textbf{95.65} & \textbf{99.73} & \textbf{96.89} & \textbf{83.24} & \textbf{95.34} & \textbf{69.78} & \textbf{69.66} & \textbf{75.58} & \textbf{13.78} & \textbf{85.23} & \textbf{96.41} & \textbf{73.33} \\
\bottomrule
\end{tabular*}
\end{table}

\paragraph{Recovery and Detection Withstand Token Deletion}
On C4, ClockMark retains 97.89\% BA and 99.33\% TPR after 10\% deletion. At 50\% deletion, BA remains 81.03\% and TPR 59.56\%, compared with 67.03\% and 18.44\% for MPAC, the strongest baseline for both metrics. The advantage therefore persists when half the tokens are removed. Surviving adjacent pairs remain usable for recovery and detection without their original positions.

\paragraph{Recovery Remains Stronger Under Substitution}
Substitution causes larger losses than deletion at the same rate. At 30\% substitution, C4 BA is 83.24\% and TPR is 69.78\%, compared with 92.58\% and 96.22\% after deletion. Nevertheless, ClockMark leads the baselines at every substitution rate. At 50\%, BA remains 69.66\% versus 59.47\% for MPAC, while TPR is 13.78\% versus 6.44\% for WeaveMark.

\FloatBarrier
\paragraph{Robustness Extends to Paraphrasing}
After DIPPER rewriting, ClockMark reaches 85.23\% BA on C4, exceeding PURA by 16.08 percentage points. AUC and TPR reach 96.41\% and 73.33\%, compared with 89.18\% and 48.00\% for BiMark, the strongest detection baseline. The gains in both recovery and detection extend the advantage beyond direct token edits.

\FloatBarrier
\section{Conclusion}
We presented an adaptive and host-dependent framework for robust, distribution-preserving multi-bit LLM watermarking. Its implementation, ClockMark, spreads all message bits across tokens and retains sampling progress to continue embedding at repeated one-token contexts. Probability-dependent sampling and state updates preserve the LLM output distribution, while a shared sum of reconstructed token-pair vectors supports recovery and detection. Experiments against twelve baselines show higher bit accuracy under the evaluated editing attacks while maintaining generation quality. These results support combining adaptive embedding with distribution preservation for robust text attribution.

\label{main-text-end}
\bibliography{references}

\begin{thebibliography}{36}
\providecommand{\natexlab}[1]{#1}
\providecommand{\url}[1]{\texttt{#1}}
\expandafter\ifx\csname urlstyle\endcsname\relax
  \providecommand{\doi}[1]{doi: #1}\else
  \providecommand{\doi}{doi: \begingroup \urlstyle{rm}\Url}\fi

\bibitem[Aaronson(2023)]{aaronson2023watermarking}
Scott Aaronson.
\newblock Watermarking of large language models.
\newblock Invited talk, Simons Institute for the Theory of Computing, August
  17, 2023.
\newblock URL
  \url{https://simons.berkeley.edu/talks/scott-aaronson-ut-austin-openai-2023-08-17}.
\newblock Accessed September 26, 2026.

\bibitem[Agrawal(2026)]{agrawal2026mosaic}
Naman Agrawal.
\newblock {MOSAIC}: Set-based multi-bit watermarking for language models.
\newblock Author-uploaded preprint, 2026.
\newblock URL
  \url{https://www.researchgate.net/publication/410836411_MOSAIC_Set-Based_Multi-Bit_Watermarking_for_Language_Models}.

\bibitem[Anderson(2007)]{anderson2007modified}
David~F Anderson.
\newblock A modified next reaction method for simulating chemical systems with
  time dependent propensities and delays.
\newblock \emph{The Journal of Chemical Physics}, 127\penalty0 (21):\penalty0
  214107, 2007.
\newblock \doi{10.1063/1.2799998}.
\newblock URL \url{https://arxiv.org/abs/0708.0370}.

\bibitem[Biderman et~al.(2023)Biderman, Schoelkopf, Anthony, Bradley, O'Brien,
  Hallahan, Khan, Purohit, Prashanth, Raff, et~al.]{pmlr-v202-biderman23a}
Stella Biderman, Hailey Schoelkopf, Quentin~Gregory Anthony, Herbie Bradley,
  Kyle O'Brien, Eric Hallahan, Mohammad~Aflah Khan, Shivanshu Purohit,
  USVSN~Sai Prashanth, Edward Raff, et~al.
\newblock Pythia: A suite for analyzing large language models across training
  and scaling.
\newblock In \emph{International Conference on Machine Learning}, pages
  2397--2430. PMLR, 2023.
\newblock URL \url{https://proceedings.mlr.press/v202/biderman23a.html}.

\bibitem[Chen et~al.(2025)Chen, Wu, Guo, and Huang]{chen2025mcmark}
Ruibo Chen, Yihan Wu, Junfeng Guo, and Heng Huang.
\newblock Improved unbiased watermark for large language models.
\newblock In \emph{Proceedings of the 63rd Annual Meeting of the Association
  for Computational Linguistics (Volume 1: Long Papers)}, pages 20587--20601,
  2025.
\newblock \doi{10.18653/v1/2025.acl-long.1005}.
\newblock URL \url{https://aclanthology.org/2025.acl-long.1005/}.

\bibitem[Cui et~al.(2026)Cui, Chen, Wu, and Huang]{cui2026mc2mark}
Xuehao Cui, Ruibo Chen, Yihan Wu, and Heng Huang.
\newblock {MC$^2$Mark}: Distortion-free multi-bit watermarking for long
  messages.
\newblock \emph{arXiv preprint arXiv:2602.14030}, 2026.
\newblock URL \url{https://arxiv.org/abs/2602.14030v1}.

\bibitem[Dathathri et~al.(2024)Dathathri, See, Ghaisas, Huang, McAdam, Welbl,
  Bachani, Kaskasoli, Stanforth, Matejovicova, et~al.]{dathathri2024scalable}
Sumanth Dathathri, Abigail See, Sumedh Ghaisas, Po-Sen Huang, Rob McAdam,
  Johannes Welbl, Vandana Bachani, Alex Kaskasoli, Robert Stanforth, Tatiana
  Matejovicova, et~al.
\newblock Scalable watermarking for identifying large language model outputs.
\newblock \emph{Nature}, 634\penalty0 (8035):\penalty0 818--823, 2024.
\newblock \doi{10.1038/s41586-024-08025-4}.
\newblock URL \url{https://www.nature.com/articles/s41586-024-08025-4}.

\bibitem[Fan et~al.(2019)Fan, Jernite, Perez, Grangier, Weston, and
  Auli]{fan-etal-2019-eli5}
Angela Fan, Yacine Jernite, Ethan Perez, David Grangier, Jason Weston, and
  Michael Auli.
\newblock {ELI5}: Long form question answering.
\newblock In \emph{Proceedings of the 57th annual meeting of the association
  for computational linguistics}, pages 3558--3567, 2019.
\newblock \doi{10.18653/v1/P19-1346}.
\newblock URL \url{https://aclanthology.org/P19-1346/}.

\bibitem[Feng et~al.(2025)Feng, Zhang, Zhang, Zhang, and Pan]{feng2025bimark}
Xiaoyan Feng, He~Zhang, Yanjun Zhang, Leo~Yu Zhang, and Shirui Pan.
\newblock {BiMark}: Unbiased multilayer watermarking for large language models.
\newblock In \emph{International Conference on Machine Learning}, pages
  17049--17067. PMLR, 2025.
\newblock URL \url{https://proceedings.mlr.press/v267/feng25u.html}.

\bibitem[Fernandez et~al.(2023)Fernandez, Chaffin, Tit, Chappelier, and
  Furon]{fernandez2023three}
Pierre Fernandez, Antoine Chaffin, Karim Tit, Vivien Chappelier, and Teddy
  Furon.
\newblock Three bricks to consolidate watermarks for large language models.
\newblock In \emph{2023 IEEE International Workshop on Information Forensics
  and Security (WIFS)}, pages 1--6. IEEE, 2023.

\bibitem[{Gemma Team} et~al.(2024){Gemma Team}, Riviere, Pathak, Sessa, Hardin,
  Bhupatiraju, Hussenot, Mesnard, Shahriari, Ram{\'e},
  et~al.]{gemmateam2024gemma2improvingopen}
{Gemma Team}, Morgane Riviere, Shreya Pathak, Pier~Giuseppe Sessa, Cassidy
  Hardin, Surya Bhupatiraju, L{\'e}onard Hussenot, Thomas Mesnard, Bobak
  Shahriari, Alexandre Ram{\'e}, et~al.
\newblock {Gemma 2}: Improving open language models at a practical size.
\newblock \emph{arXiv preprint arXiv:2408.00118}, 2024.
\newblock URL \url{https://arxiv.org/abs/2408.00118}.

\bibitem[Gilani et~al.(2026)Gilani, Vithana, Long, Kosut, Sankar, and
  Calmon]{gilani2026arcmark}
Atefeh Gilani, Sajani Vithana, Carol~Xuan Long, Oliver Kosut, Lalitha Sankar,
  and Flavio~P Calmon.
\newblock {ArcMark}: Distortion-free multi-byte {LLM} watermark via optimal
  transport.
\newblock \emph{arXiv preprint arXiv:2602.07235}, 2026.
\newblock URL \url{https://arxiv.org/abs/2602.07235v2}.

\bibitem[Gloaguen et~al.(2026)Gloaguen, Staab, Vero, and
  Vechev]{gloaguen2026everybit}
Thibaud Gloaguen, Robin Staab, Mark Vero, and Martin Vechev.
\newblock Every bit, everywhere, all at once: A binomial multibit {LLM}
  watermark.
\newblock \emph{arXiv preprint arXiv:2605.11653}, 2026.
\newblock URL \url{https://arxiv.org/abs/2605.11653v1}.

\bibitem[Hu et~al.(2024)Hu, Chen, Wu, Wu, Zhang, and Huang]{hu2024unbiased}
Zhengmian Hu, Lichang Chen, Xidong Wu, Yihan Wu, Hongyang Zhang, and Heng
  Huang.
\newblock Unbiased watermark for large language models.
\newblock In \emph{International Conference on Learning Representations}, 2024.
\newblock URL \url{https://arxiv.org/abs/2310.10669}.

\bibitem[Jakesch et~al.(2023)Jakesch, Hancock, and Naaman]{jakesch2023human}
Maurice Jakesch, Jeffrey~T Hancock, and Mor Naaman.
\newblock Human heuristics for {AI}-generated language are flawed.
\newblock \emph{Proceedings of the National Academy of Sciences}, 120\penalty0
  (11):\penalty0 e2208839120, 2023.
\newblock \doi{10.1073/pnas.2208839120}.

\bibitem[Jiang et~al.(2026{\natexlab{a}})Jiang, Lei, Wang, Wang, and
  Yi]{jiang2024credid}
Haoyu Jiang, Shanzhe Lei, Yingchun Wang, Xuhong Wang, and Ping Yi.
\newblock {CredID}: Credible multi-bit watermark for large language models
  identification.
\newblock In \emph{ICASSP 2026-2026 IEEE International Conference on Acoustics,
  Speech and Signal Processing (ICASSP)}, pages 13962--13966. IEEE,
  2026{\natexlab{a}}.
\newblock URL
  \url{https://www.cmsworkshops.com/ICASSP2026/view_paper.php?PaperNum=8760}.

\bibitem[Jiang et~al.(2025)Jiang, Wu, Boroujeny, Mark, and
  Zeng]{jiang2025stealthink}
Ya~Jiang, Chuxiong Wu, Massieh~Kordi Boroujeny, Brian Mark, and Kai Zeng.
\newblock {StealthInk}: A multi-bit and stealthy watermark for large language
  models.
\newblock In \emph{International Conference on Machine Learning}, pages
  27685--27709. PMLR, 2025.
\newblock URL \url{https://proceedings.mlr.press/v267/jiang25j.html}.

\bibitem[Jiang et~al.(2026{\natexlab{b}})Jiang, Boroujeny, Kumar, and
  Zeng]{jiang2026mirrormark}
Ya~Jiang, Massieh~Kordi Boroujeny, Surender~Suresh Kumar, and Kai Zeng.
\newblock {MirrorMark}: Generalizable mirrored sampling for multi-bit {LLM}
  watermarking.
\newblock \emph{arXiv preprint arXiv:2601.22246}, 2026{\natexlab{b}}.
\newblock URL \url{https://arxiv.org/abs/2601.22246v3}.

\bibitem[Kirchenbauer et~al.(2023)Kirchenbauer, Geiping, Wen, Katz, Miers, and
  Goldstein]{kirchenbauer2023watermark}
John Kirchenbauer, Jonas Geiping, Yuxin Wen, Jonathan Katz, Ian Miers, and Tom
  Goldstein.
\newblock A watermark for large language models.
\newblock In \emph{International Conference on Machine Learning}, pages
  17061--17084. PMLR, 2023.
\newblock URL \url{https://proceedings.mlr.press/v202/kirchenbauer23a.html}.

\bibitem[Krishna et~al.(2023)Krishna, Song, Karpinska, Wieting, and
  Iyyer]{krishna2023dipper}
Kalpesh Krishna, Yixiao Song, Marzena Karpinska, John Wieting, and Mohit Iyyer.
\newblock Paraphrasing evades detectors of {AI}-generated text, but retrieval
  is an effective defense.
\newblock In \emph{Advances in Neural Information Processing Systems},
  volume~36, pages 27469--27500, 2023.
\newblock URL \url{https://arxiv.org/abs/2303.13408}.

\bibitem[Kuditipudi et~al.(2024)Kuditipudi, Thickstun, Hashimoto, and
  Liang]{kuditipudi2024distortionfree}
Rohith Kuditipudi, John Thickstun, Tatsunori Hashimoto, and Percy Liang.
\newblock Robust distortion-free watermarks for language models.
\newblock \emph{Transactions on Machine Learning Research}, 2024.
\newblock URL \url{https://openreview.net/forum?id=FpaCL1MO2C}.

\bibitem[Maddison(2016)]{maddison2016poisson}
Chris~J Maddison.
\newblock A {Poisson} process model for {Monte Carlo}.
\newblock In Tamir Hazan, George Papandreou, and Daniel Tarlow, editors,
  \emph{Perturbations, Optimization, and Statistics}, pages 193--232. MIT
  Press, 2016.
\newblock URL \url{https://arxiv.org/abs/1602.05986}.

\bibitem[Moulin and O'Sullivan(2003)]{moulin2003information}
Pierre Moulin and Joseph~A O'Sullivan.
\newblock Information-theoretic analysis of information hiding.
\newblock \emph{IEEE Transactions on information theory}, 49\penalty0
  (3):\penalty0 563--593, 2003.

\bibitem[Park et~al.(2026{\natexlab{a}})Park, Lee, Kwak, and
  Yun]{park2026weavemark}
Gang-Hyun Park, Ju-Hyeong Lee, Hee-Youl Kwak, and Dae-Young Yun.
\newblock {WeaveMark}: Robust and scalable multi-bit {LLM} watermarking via
  coded payload spreading.
\newblock \emph{arXiv preprint arXiv:2609.02177}, 2026{\natexlab{a}}.
\newblock URL \url{https://arxiv.org/abs/2609.02177v1}.

\bibitem[Park et~al.(2026{\natexlab{b}})Park, Park, Ahn, and
  Han]{park2026watermod}
Shinwoo Park, Hyejin Park, Hyeseon Ahn, and Yo-Sub Han.
\newblock {WaterMod}: Modular token-rank partitioning for probability-balanced
  {LLM} watermarking.
\newblock \emph{Proceedings of the AAAI Conference on Artificial Intelligence},
  40\penalty0 (39):\penalty0 32683--32691, 2026{\natexlab{b}}.
\newblock \doi{10.1609/aaai.v40i39.40546}.
\newblock URL \url{https://ojs.aaai.org/index.php/AAAI/article/view/40546}.

\bibitem[Paszke et~al.(2019)Paszke, Gross, Massa, Lerer, Bradbury, Chanan,
  Killeen, Lin, Gimelshein, Antiga, et~al.]{NEURIPS2019_bdbca288}
Adam Paszke, Sam Gross, Francisco Massa, Adam Lerer, James Bradbury, Gregory
  Chanan, Trevor Killeen, Zeming Lin, Natalia Gimelshein, Luca Antiga, et~al.
\newblock {PyTorch}: An imperative style, high-performance deep learning
  library.
\newblock In \emph{Advances in Neural Information Processing Systems},
  volume~32, 2019.
\newblock URL
  \url{https://proceedings.neurips.cc/paper_files/paper/2019/file/bdbca288fee7f92f2bfa9f7012727740-Paper.pdf}.

\bibitem[Porter and Machery(2024)]{porter2024poetry}
Brian Porter and Edouard Machery.
\newblock {AI}-generated poetry is indistinguishable from human-written poetry
  and is rated more favorably.
\newblock \emph{Scientific Reports}, 14\penalty0 (1):\penalty0 26133, 2024.
\newblock \doi{10.1038/s41598-024-76900-1}.

\bibitem[Qu et~al.(2025)Qu, Zheng, Tao, Yin, Jiang, Tian, Zou, Jia, and
  Zhang]{qu2025robust}
Wenjie Qu, Wengrui Zheng, Tianyang Tao, Dong Yin, Yanze Jiang, Zhihua Tian, Wei
  Zou, Jinyuan Jia, and Jiaheng Zhang.
\newblock Provably robust multi-bit watermarking for {AI}-generated text.
\newblock In \emph{34th USENIX Security Symposium (USENIX Security 25)}, pages
  201--220, 2025.
\newblock URL
  \url{https://www.usenix.org/conference/usenixsecurity25/presentation/qu-watermarking}.

\bibitem[{Qwen} et~al.(2024){Qwen}, Yang, Yang, Zhang, Hui, Zheng, Yu, Li, Liu,
  Huang, Wei, et~al.]{qwen2024qwen25}
{Qwen}, An~Yang, Baosong Yang, Beichen Zhang, Binyuan Hui, Bo~Zheng, Bowen Yu,
  Chengyuan Li, Dayiheng Liu, Fei Huang, Haoran Wei, et~al.
\newblock {Qwen2.5} technical report.
\newblock \emph{arXiv preprint arXiv:2412.15115}, 2024.
\newblock URL \url{https://arxiv.org/abs/2412.15115v2}.
\newblock Revised January 3, 2025, version 2.

\bibitem[Raffel et~al.(2020)Raffel, Shazeer, Roberts, Lee, Narang, Matena,
  Zhou, Li, and Liu]{JMLR:v21:20-074}
Colin Raffel, Noam Shazeer, Adam Roberts, Katherine Lee, Sharan Narang, Michael
  Matena, Yanqi Zhou, Wei Li, and Peter~J Liu.
\newblock Exploring the limits of transfer learning with a unified text-to-text
  transformer.
\newblock \emph{Journal of Machine Learning Research}, 21\penalty0
  (140):\penalty0 1--67, 2020.
\newblock URL \url{http://jmlr.org/papers/v21/20-074.html}.

\bibitem[Wang et~al.(2024)Wang, Yang, Chen, Zhou, Lin, Meng, Zhou, and
  Sun]{wang2024codable}
Lean Wang, Wenkai Yang, Deli Chen, Hao Zhou, Yankai Lin, Fandong Meng, Jie
  Zhou, and Xu~Sun.
\newblock Towards codable watermarking for injecting multi-bits information to
  {LLMs}.
\newblock In \emph{International Conference on Learning Representations}, 2024.
\newblock URL \url{https://openreview.net/forum?id=JYu5Flqm9D}.

\bibitem[Wang et~al.(2026)Wang, Guo, Du, Wang, Yao, Hu, Zhang, Yu, and
  Chen]{wang2026pura}
Yaofei Wang, Jinyang Guo, Shuchao Du, Chao Wang, Qiyi Yao, Donghui Hu, Weiming
  Zhang, Nenghai Yu, and Kejiang Chen.
\newblock {PURA}: Provably unbiased and robust multi-bit watermarking for
  {AI}-generated text attribution.
\newblock \emph{arXiv preprint arXiv:2608.22218}, 2026.
\newblock URL \url{https://arxiv.org/abs/2608.22218v2}.
\newblock Accepted to ACM CCS 2026; arXiv version 2.

\bibitem[Wu et~al.(2024)Wu, Hu, Guo, Zhang, and Huang]{wu2024dipmark}
Yihan Wu, Zhengmian Hu, Junfeng Guo, Hongyang Zhang, and Heng Huang.
\newblock A resilient and accessible distribution-preserving watermark for
  large language models.
\newblock In \emph{International Conference on Machine Learning}, pages
  53443--53470. PMLR, 2024.
\newblock URL \url{https://proceedings.mlr.press/v235/wu24h.html}.

\bibitem[Xu et~al.(2026)Xu, Hu, Kotevska, and Zhang]{xu2026xmark}
Jiahao Xu, Rui Hu, Olivera Kotevska, and Zikai Zhang.
\newblock {XMark}: Reliable multi-bit watermarking for {LLM}-generated texts.
\newblock In \emph{Proceedings of the 64th Annual Meeting of the Association
  for Computational Linguistics (Volume 1: Long Papers)}, pages 14747--14763,
  2026.
\newblock \doi{10.18653/v1/2026.acl-long.672}.
\newblock URL \url{https://aclanthology.org/2026.acl-long.672/}.

\bibitem[Yoo et~al.(2024)Yoo, Ahn, and Kwak]{yoo2024mpac}
KiYoon Yoo, Wonhyuk Ahn, and Nojun Kwak.
\newblock Advancing beyond identification: Multi-bit watermark for large
  language models.
\newblock In \emph{Proceedings of the 2024 Conference of the North American
  Chapter of the Association for Computational Linguistics: Human Language
  Technologies (Volume 1: Long Papers)}, pages 4031--4055, 2024.
\newblock \doi{10.18653/v1/2024.naacl-long.224}.
\newblock URL \url{https://aclanthology.org/2024.naacl-long.224/}.

\bibitem[Zhu et~al.(2026)Zhu, Huang, and Wan]{zhu2026quantilemark}
Junlin Zhu, Baizhou Huang, and Xiaojun Wan.
\newblock {QuantileMark}: A message-symmetric multi-bit watermark for {LLMs}.
\newblock In \emph{Proceedings of the 64th Annual Meeting of the Association
  for Computational Linguistics (Volume 1: Long Papers)}, pages 6790--6806,
  2026.
\newblock \doi{10.18653/v1/2026.acl-long.308}.
\newblock URL \url{https://aclanthology.org/2026.acl-long.308/}.

\end{thebibliography}
\bibliographystyle{plainnat}
\clearpage
\appendix
\section{Watermarking Algorithms}
\label{app:algorithm}\label{app:implementation}
\begingroup
\raggedbottom

\begin{algorithm}[H]
\caption{ClockMark watermark embedding}\label{alg:sender}
\begin{algorithmic}[1]
\Require LLM $M$ with a fixed input prompt, message $m\in\{0,1\}^b$, secret key $K\in\mathcal K$, generation length $T$
\Ensure Watermarked token sequence $x_{1:T}$
\State $y\gets2m-\mathbf1$, $u\gets y/\sqrt b$ \Comment{Message spreading}
\State Initialize an empty context memory storing $(R_{c,v},n_{c,v})$
\State Sample $x_1\sim P_M(\cdot\mid x_{<1})$
\For{$t=2,\ldots,T$}
 \State $c\gets x_{t-1}$
 \State $p_t(\cdot)\gets P_M(\cdot\mid x_{<t})$
 \For{$v\in\mathcal V$ with $p_t(v)>0$}
  \If{$(c,v)$ has no saved state}
   \State $n_{c,v}\gets0$
   \State Generate $z_{c,v,0}\in\R^b$ from $K$, $(c,v)$, and counter $0$
   \State $R_{c,v}\gets E_{c,v,0}=-\ln\Phi(z_{c,v,0}^{\top}u)$
  \Else
   \State Load the saved $(R_{c,v},n_{c,v})$ \Comment{Resume sampling progress}
  \EndIf
 \EndFor
 \State $x_t\gets\argmin_{v:p_t(v)>0}R_{c,v}/p_t(v)$ \Comment{Host-dependent race}
 \State $\Delta_t\gets R_{c,x_t}/p_t(x_t)$
 \For{$v\in\mathcal V$ with $p_t(v)>0$}
  \State $R_{c,v}\gets R_{c,v}-p_t(v)\Delta_t$ \Comment{Update all active countdowns}
 \EndFor
 \State $n_{c,x_t}\gets n_{c,x_t}+1$
 \State Generate $z_{c,x_t,n_{c,x_t}}\in\R^b$ from $K$, $(c,x_t)$, and counter $n_{c,x_t}$
 \State $R_{c,x_t}\gets E_{c,x_t,n_{c,x_t}}=-\ln\Phi(z_{c,x_t,n_{c,x_t}}^{\top}u)$ \Comment{Renew only the winner}
 \State Save the updated $(R_{c,v},n_{c,v})$ for context $c$ \Comment{Adaptive memory}
\EndFor
\State \Return $x_{1:T}$
\end{algorithmic}
\end{algorithm}

\begin{algorithm}[H]
\caption{ClockMark watermark recovery (VSR)}\label{alg:recovery}
\begin{algorithmic}[1]
\Require Received token sequence $\widetilde x_{1:\widetilde T}$, secret key $K\in\mathcal K$, message length $b$
\Ensure Recovered message $\widehat m$, vector sum $S\in\R^b$, distinct pair count $N$
\State $\mathcal P\gets\{(\widetilde x_{t-1},\widetilde x_t):2\le t\le\widetilde T\}$ \Comment{Count each pair once}
\State $N\gets|\mathcal P|$, $S\gets\mathbf0\in\R^b$
\For{$(c,v)\in\mathcal P$}
 \State Regenerate $z_{c,v,0}\in\R^b$ from $K$, $(c,v)$, and counter $0$
 \State $z_{c,v}\gets z_{c,v,0}$
 \State $S\gets S+z_{c,v}$
\EndFor
\State $\widehat y\gets\sgn(S)$, with $\sgn(0)=+1$ coordinatewise
\State $\widehat m\gets(\widehat y+\mathbf1)/2$
\State \Return $\widehat m$, $S$, $N$
\end{algorithmic}
\end{algorithm}

\clearpage
\begin{algorithm}[H]
\caption{ClockMark watermark detection (ED)}\label{alg:detection}
\begin{algorithmic}[1]
\Require Vector sum $S$ and distinct pair count $N>0$ from Algorithm~\ref{alg:recovery}, message length $b$, significance level $\alpha\in(0,1)$
\Ensure Detection decision and p-value $p_E$
\State $E_{\rm det}\gets\|S\|^2/N$
\State $p_E\gets\Prb_{Q\sim\chi_b^2}(Q\ge E_{\rm det})$
\If{$p_E\le\alpha$}
 \State \Return Watermark detected, $p_E$
\Else
 \State \Return Watermark not detected, $p_E$
\EndIf
\end{algorithmic}
\end{algorithm}
\endgroup

\section{Distribution Preservation}
\label{app:theory}
We prove Theorem~\ref{thm:preservation} for a fixed prompt, message $m\in\{0,1\}^b$, and generation length $T$, chosen independently of the key.

\subsection{Proof Idea and Invariant}
\label{app:proof-idea}
The proof connects token selection to sequence preservation through one invariant: conditional on previously selected tokens and elapsed race times, all stored remaining countdowns and all unused intervals are mutually independent $\Exp(1)$ variables. Here $\Exp(1)$ is the exponential distribution with rate one and mean one. This law ensures that each candidate wins with its current LLM probability. Showing that the countdown update preserves it then justifies reusing saved progress throughout generation.

The remaining amount $R_{c,v}$ is an internal countdown, whereas $R_{c,v}/p_t(v)$ is its completion time at the current positive rate. This distinction allows countdowns to pause and resume at different rates, as in the internal-time construction of \citet{anderson2007modified}.

\subsection{Exponential-Race Invariance}
\label{app:race-invariance}
Under the independent Gaussian model in Section~\ref{sec:intervals}, the unit message direction $u=(2m-\mathbf1)/\sqrt b$ gives $z_{c,v,n}^{\top}u\sim\mathcal N(0,1)$. The probability integral transform therefore yields
\begin{equation}
 \Phi(z_{c,v,n}^{\top}u)\sim\operatorname{Unif}(0,1),
 \qquad E_{c,v,n}=-\ln\Phi(z_{c,v,n}^{\top}u)\sim\Exp(1).
 \label{eq:initial-exponential-law}
\end{equation}
Intervals for distinct $(c,v,n)$ are independent, establishing the law for initial and fresh countdowns.

For a single race, index candidates with positive probability by $i$, with $\sum_i p_i=1$. If $R_i$ are independent $\Exp(1)$ variables, then $R_i/p_i\sim\Exp(p_i)$ has rate $p_i$. The standard exponential-race identity gives winner probability $p_i$ \citep[Lemma~7.6]{maddison2016poisson}. To verify that saving the unselected countdowns preserves their joint law, define
\begin{equation}
 J=\argmin_i\frac{R_i}{p_i},\qquad
 \Delta=\min_i\frac{R_i}{p_i},\qquad
 R_i'=R_i-p_i\Delta\quad(i\ne J).
 \label{eq:single-race-variables}
\end{equation}
For winner $j$, elapsed time $\delta>0$, and residual values $r_i'>0$, write $R_j=p_j\delta$ and $R_i=p_i\delta+r_i'$ for $i\ne j$. This transformation has absolute Jacobian $p_j$, and $\sum_iR_i=\delta+\sum_{i\ne j}r_i'$. The joint law, with probability mass in $j$, factors as
\begin{equation}
 f(j,\delta,(r_i')_{i\ne j})
 =\underbrace{p_j}_{\text{winner}}\,
  \underbrace{e^{-\delta}}_{\text{race time}}\,
  \underbrace{\prod_{i\ne j}e^{-r_i'}}_{\text{unselected residuals}}.
 \label{eq:race-factorization}
\end{equation}
The factors give the winner probabilities, an independent $\Exp(1)$ race time, and independent $\Exp(1)$ residuals conditional on both the winner and time. This is the joint residual invariance needed for reuse. Giving the winner a fresh independent interval restores the product exponential law for the entire updated vector, independently of $(J,\Delta)$.

\subsection{Sequence-Level Preservation}
\label{app:preservation}
Let $X_t$ denote the token sampled at step $t$. The history $\mathcal H_{t-1}$ contains the tokens and race times observed before step $t$, with the prompt and message fixed. It excludes the key, residual values, and unused vectors. Conditional on this history, the context and probabilities $p_t(v)=P_M(v\mid X_{<t})$ are fixed, while the remaining countdowns are random.

\begin{proof}[Proof of Theorem~\ref{thm:preservation}]
The first token is sampled from $P_M$ independently of the keyed vectors. Initial countdowns therefore establish the invariant before the first race. Assume this invariant holds conditional on $\mathcal H_{t-1}$. Newly initialized pairs receive fresh initial intervals. The single-race calculation gives
\begin{equation}
 \Prb(X_t=v\mid\mathcal H_{t-1})=p_t(v).
 \label{eq:conditional-token-law}
\end{equation}
Equation~\eqref{eq:race-factorization} and renewal of the winner preserve the residual law after conditioning on the new token and race time. The race depends only on countdowns with positive probability in the current context, so all other countdowns remain unchanged and conditionally independent of its outcome. Unused intervals remain independent as well, establishing the invariant conditional on $\mathcal H_t$.

This argument also covers a returning context with different LLM probabilities: its saved residuals still have the required conditional law, so the next race uses the new probabilities correctly. Averaging over race times while keeping the generated prefix fixed yields
\begin{equation}
 \Prb(X_t=v\mid X_{<t}=x_{<t},m)=P_M(v\mid x_{<t}).
 \label{eq:prefix-token-law}
\end{equation}
The chain rule then gives
\begin{equation}
 \E_{K\sim P_K}\!\left[P_{M,w}(x_{1:T}\mid m,K)\right]
 =\Prb(X_{1:T}=x_{1:T}\mid m)
 =\prod_{t=1}^T P_M(x_t\mid x_{<t}),
 \label{eq:sequence-chain-rule}
\end{equation}
which is Equation~\eqref{eq:preserve}.
\end{proof}

The token-level result establishes host-dependent selection, while the sequence-level argument justifies adaptive reuse of sampling progress. Together, they support continued embedding at repeated one-token contexts, limiting the spread of edit damage without discarding embedding opportunities.

\section{Energy Detection}
\label{app:receiver}\label{app:energy}
This appendix derives the null distribution and false-positive control used in Section~\ref{sec:detection}.

\paragraph{Detection Intuition.}
Embedding favors token-pair vectors aligned with the message direction, motivating detection through the magnitude of their sum $S$. Whereas VSR uses coordinate signs to recover bits, ED uses $E_{\rm det}=\|S\|^2/N$ to measure their combined strength. Vector length does not require knowing the message direction, so detection uses the same watermark signals without knowing the embedded message.

\paragraph{Exact Null Distribution.}
Under $H_0$, the received text $\widetilde x$ is independent of the secret key. We first fix this text, hence its distinct pair set $\mathcal P$ and count $N>0$, and consider randomness only over $K\sim P_K$. Independence ensures that conditioning on the text does not change the key distribution. Under the Gaussian construction in Section~\ref{sec:intervals}, the reconstructed vectors $z_{c,v,0}$ for these fixed, distinct pairs therefore remain independent $\mathcal N(0,I_b)$ variables. Consequently,
\begin{equation}
 S=\sum_{(c,v)\in\mathcal P}z_{c,v,0},
 \qquad S\mid\widetilde x\sim\mathcal N(0,NI_b).
 \label{eq:null-vector-sum}
\end{equation}
Each coordinate has variance $N$. Dividing by $\sqrt N$ gives $b$ independent standard normal coordinates, whose squared sum follows a chi-square distribution:
\begin{equation}
 \frac{S}{\sqrt N}\,\bigg|\,\widetilde x\sim\mathcal N(0,I_b),
 \qquad
 E_{\rm det}\mid\widetilde x\sim\chi_b^2.
 \label{eq:energy-null}
\end{equation}
The degrees of freedom are $b$, while normalization removes the dependence of the null distribution on the pair count $N$.

\paragraph{Distinct Pairs.}
Repeated occurrences of a pair reconstruct the same vector. Counting them separately would introduce perfectly correlated copies and invalidate the covariance $NI_b$ used above. Counting each distinct pair once is therefore essential for this normalization and chi-square calibration.

\paragraph{False-Positive Control.}
Let $F_{\chi_b^2}$ denote the cumulative distribution function of $\chi_b^2$. The upper-tail p-value in Section~\ref{sec:detection} can be written as
\begin{equation}
 p_E=1-F_{\chi_b^2}(E_{\rm det}).
 \label{eq:energy-pvalue-cdf}
\end{equation}
Because Equation~\eqref{eq:energy-null} gives a continuous null distribution, the probability integral transform implies $p_E\mid\widetilde x\sim\operatorname{Unif}(0,1)$ under $H_0$. Thresholding at any $\alpha\in(0,1)$ therefore gives
\begin{equation}
 \Prb_{K\sim P_K}(p_E\le\alpha\mid\widetilde x)=\alpha
 \quad\text{when }N>0.
 \label{eq:energy-level}
\end{equation}
Equivalently, detection uses $E_{\rm det}\ge F_{\chi_b^2}^{-1}(1-\alpha)$, a threshold determined by $b$ and $\alpha$, independent of $N$. For any distribution of received texts independent of $K$ and with $N>0$, the joint false-positive probability over the text and $K\sim P_K$ is therefore $\alpha$, even when text lengths vary.

\section{Experimental Protocol}
\label{app:protocol}
\subsection{Baseline Configurations}
\label{app:baselines}
For a fair comparison, the strength parameters of biased baselines are selected with the aim of keeping text quality at a comparable level. Table~\ref{tab:baseline-configurations} lists the embedding hyperparameters, which are fixed across payload sizes and attack conditions.
\begin{table}[H]
\papertablestyle
\setlength{\tabcolsep}{4pt}
\caption{Embedding configurations of the twelve baselines, grouped by biased and unbiased embedding. Strength $\delta$ and symbol parameter $m$ follow the definitions of each baseline. SPC(5,4) is a single parity-check code with four data bits and one check bit.}\label{tab:baseline-configurations}
\begin{tabular*}{\textwidth}{@{\extracolsep{\fill}}cll@{}}
\toprule
Category & \multicolumn{1}{c}{Method} & \multicolumn{1}{c}{Embedding Setting} \\
\midrule
\multirow{7}{*}{Biased} & BalanceMarking & $\delta=1.5$, 8-bit segments \\
 & CredID & $\delta=1.5$, 8-bit segments \\
 & MOSAIC & $\delta=1.5$ \\
 & MPAC & $\delta=2$, radix 4 \\
 & RSBH & $\delta=1$ \\
 & WaterMod & $\delta=2.5$ \\
 & XMark & $\delta=2$ \\
\midrule
\multirow{5}{*}{Unbiased} & BiMark & $\delta=1$, 10 layers \\
 & StealthInk & $m=1$ \\
 & PURA & SPC(5,4) coding \\
 & QuantileMark & $m=2$ \\
 & WeaveMark & $\delta=1$, 8 message layers + 2 detection layers \\
\bottomrule
\end{tabular*}
\end{table}

\FloatBarrier
\section{Additional Evaluation}
\label{app:full-results}
This appendix extends Section~\ref{sec:experiments} with three comparisons against the same twelve baselines. Appendix~\ref{app:clean-lfqa} evaluates text quality, recovery, and detection on LFQA using Qwen2.5-7B-Instruct with 32 and 64 bit payloads. Appendix~\ref{app:gemma} evaluates the same metrics with Gemma-2-9B Base on both LFQA and C4 at both payload lengths. Each clean setting contains 500 texts of 200 tokens per method. Appendix~\ref{app:robust-lfqa} evaluates LFQA robustness using 500 matched Qwen texts with 32 bit payloads, each originally containing 200 tokens. The seven attack conditions comprise deletion and substitution at rates of 10\%, 30\%, and 50\%, together with DIPPER paraphrasing. Baseline configurations follow Appendix~\ref{app:baselines}, and evaluation metrics follow Section~\ref{sec:experiments}. Detection uses 5000 unwatermarked texts per model and dataset.

\subsection{Text Quality, Recovery, and Detection on LFQA}
\label{app:clean-lfqa}
Table~\ref{tab:clean-lfqa} reports the Qwen comparison on unedited LFQA texts, complementing the C4 results in Table~\ref{tab:clean}.
\begin{table}[H]
\papertablestyle
\setlength{\tabcolsep}{4pt}
\caption{Clean Qwen results on LFQA with 32 and 64 bit payloads. Bit accuracy (BA), area under the ROC curve (AUC), and true positive rate (TPR) at a 1\% false positive rate (FPR) are percentages. Bold marks the best displayed BA, AUC, and TPR. Baselines without reported detection appear first. Dashes indicate unavailable or inapplicable results.}\label{tab:clean-lfqa}
\begin{tabular*}{\textwidth}{@{\extracolsep{\fill}}lrrrrrrrr@{}}
\toprule
\multicolumn{1}{c}{\multirow{2}{*}{Method}} & \multicolumn{4}{c}{\textbf{LFQA (32 bits)}} & \multicolumn{4}{c}{\textbf{LFQA (64 bits)}} \\
\cmidrule(lr){2-5}\cmidrule(lr){6-9}
 & PPL $\downarrow$ & BA $\uparrow$ & AUC $\uparrow$ & TPR $\uparrow$ & PPL $\downarrow$ & BA $\uparrow$ & AUC $\uparrow$ & TPR $\uparrow$ \\
\midrule
No watermark & 11.05 & -- & -- & -- & 11.05 & -- & -- & -- \\
\midrule
BalanceMarking & 11.13 & 62.78 & -- & -- & 11.25 & 55.63 & -- & -- \\
CredID & 11.56 & 62.14 & -- & -- & 11.54 & 55.19 & -- & -- \\
PURA & 11.29 & 85.05 & -- & -- & 11.21 & 73.34 & -- & -- \\
RSBH & 11.56 & 52.62 & -- & -- & 11.61 & 50.92 & -- & -- \\
WaterMod & 14.34 & 74.34 & -- & -- & 14.58 & 72.19 & -- & -- \\
\midrule
BiMark & 11.11 & 84.74 & 99.74 & 96.60 & 11.12 & 69.27 & 99.18 & 90.80 \\
MOSAIC & 12.04 & 52.51 & 50.37 & 1.20 & 12.20 & 50.97 & 50.00 & 0.00 \\
MPAC & 12.96 & 79.74 & 95.52 & 67.20 & 12.86 & 70.62 & 91.76 & 45.80 \\
QuantileMark & 10.99 & 77.95 & 90.92 & 48.60 & 11.06 & 69.47 & 82.93 & 16.60 \\
StealthInk & 11.21 & 77.25 & 90.77 & 35.80 & 11.00 & 69.81 & 82.09 & 21.20 \\
WeaveMark & 11.14 & 71.03 & 99.03 & 85.20 & 11.11 & 69.46 & 98.91 & 82.20 \\
XMark & 12.72 & 79.51 & 93.09 & 48.40 & 12.91 & 70.58 & 87.75 & 21.80 \\
\midrule
\textcolor{blue!45!black}{ClockMark} & 10.97 & \textbf{93.42} & \textbf{99.95} & \textbf{98.80} & 11.07 & \textbf{85.98} & \textbf{99.80} & \textbf{97.20} \\
\bottomrule
\end{tabular*}
\end{table}

\paragraph{Larger Payloads Widen the Recovery Lead}
ClockMark achieves 93.42\% BA at 32 bits and 85.98\% at 64 bits, compared with 85.05\% and 73.34\% for PURA, the strongest recovery baseline. Although BA decreases as the payload doubles, the lead grows from 8.37 to 12.64 percentage points. The same trend appears on C4, showing that the recovery advantage at larger payloads holds across both datasets.

\paragraph{Detection Remains Strong at Both Payload Lengths}
TPR decreases from 98.80\% to 97.20\% as the payload increases from 32 to 64 bits, whereas BiMark, the strongest detection baseline, decreases from 96.60\% to 90.80\%. The TPR advantage therefore grows from 2.20 to 6.40 percentage points. AUC remains 99.95\% and 99.80\% at the two payload lengths, indicating strong separation between watermarked and unwatermarked texts even as recovery becomes more difficult.

\paragraph{Perplexity Remains Close to Unwatermarked Text}
PPL is 10.97 and 11.07 at 32 and 64 bits, respectively, compared with 11.05 for unwatermarked text. These values are also close to those of BiMark and PURA. The higher recovery and detection scores are therefore accompanied by perplexity comparable to baselines that also preserve generation quality.

\FloatBarrier
\subsection{Text Quality, Recovery, and Detection on Gemma}
\label{app:gemma}
Table~\ref{tab:gemma} compares the methods on LFQA and C4 using Gemma-2-9B Base with 32 and 64 bit payloads. Generation uses temperature 1 and top-$k=128$, without a chat template. PPL is evaluated on the generated texts and 500 unwatermarked texts for each dataset, using the same evaluator as in the Qwen experiments.

\paragraph{Recovery Leads on Both Datasets}
At 32 bits, ClockMark achieves 99.57\% BA on LFQA and 99.52\% on C4, exceeding WeaveMark, the strongest recovery baseline, at 98.87\% and 98.42\%. At 64 bits, ClockMark retains 97.33\% and 97.02\% BA, and the lead widens to 4.76 and 7.56 percentage points. The larger gains at 64 bits are consistent with the Qwen comparisons, extending the recovery advantage to another model family.

\paragraph{Detection Remains Near Perfect}
At both payload lengths, ClockMark achieves AUC above 99.99\%, with TPR of 99.80\% on LFQA and 100.00\% on C4. These TPR values match the best reported results. ClockMark thus combines detection at the level of the strongest baselines with the highest BA in all four Gemma settings.

\paragraph{Perplexity Remains Stable Across Payload Lengths}
On LFQA, PPL is 13.22 and 13.20 at 32 and 64 bits, compared with 13.42 for unwatermarked text. On C4, the corresponding values are 12.78 and 12.90 against 12.47. Doubling the payload changes PPL by at most 0.12, while recovery remains above 97\% BA. The recovery and detection gains therefore coexist with perplexity close to the unwatermarked reference across both model families.
\par
\begingroup
\papertablestyle
\setlength{\tabcolsep}{4pt}
\setlength{\LTcapwidth}{\textwidth}
\setlength{\LTleft}{0pt}\setlength{\LTright}{0pt}
\begin{longtable}{@{\extracolsep{\fill}}lrrrrrrrr@{}}
\caption{Clean Gemma-2-9B results on LFQA and C4 with 32 and 64 bit payloads. Bit accuracy (BA), area under the ROC curve (AUC), and true positive rate (TPR) at a 1\% false positive rate (FPR) are percentages. Bold marks the best displayed BA, AUC, and TPR. Baselines without reported detection appear first. Dashes indicate unavailable or inapplicable results.}\label{tab:gemma}\\
\toprule
\multicolumn{1}{c}{\multirow{2}{*}{Method}} & \multicolumn{4}{c}{\textbf{LFQA (32 bits)}} & \multicolumn{4}{c}{\textbf{C4 (32 bits)}} \\*
\cmidrule(lr){2-5}\cmidrule(lr){6-9}
 & PPL $\downarrow$ & BA $\uparrow$ & AUC $\uparrow$ & TPR $\uparrow$ & PPL $\downarrow$ & BA $\uparrow$ & AUC $\uparrow$ & TPR $\uparrow$ \\*
\midrule
No watermark & 13.42 & -- & -- & -- & 12.47 & -- & -- & -- \\*
\midrule
BalanceMarking & 14.72 & 66.91 & -- & -- & 13.84 & 65.92 & -- & -- \\*
CredID & 14.84 & 71.20 & -- & -- & 13.98 & 68.15 & -- & -- \\*
PURA & 13.16 & 69.63 & -- & -- & 12.79 & 68.96 & -- & -- \\*
RSBH & 14.28 & 57.96 & -- & -- & 13.20 & 56.21 & -- & -- \\*
WaterMod & 21.63 & 92.69 & -- & -- & 18.73 & 90.14 & -- & -- \\*
\midrule
BiMark & 13.30 & 95.96 & 99.80 & 99.60 & 12.74 & 95.53 & \textbf{100.00} & \textbf{100.00} \\*
MOSAIC & 15.73 & 60.58 & 57.70 & 16.00 & 14.36 & 56.61 & 54.08 & 8.60 \\*
MPAC & 18.12 & 91.47 & 99.85 & 99.40 & 15.86 & 89.88 & 99.93 & 97.80 \\*
QuantileMark & 13.28 & 49.95 & 47.72 & 0.60 & 12.81 & 49.24 & 47.92 & 0.60 \\*
StealthInk & 13.34 & 88.78 & 99.41 & 91.60 & 12.40 & 88.14 & 99.25 & 92.00 \\*
WeaveMark & 13.61 & 98.87 & 99.99 & \textbf{99.80} & 12.55 & 98.42 & 99.98 & 99.80 \\*
XMark & 17.04 & 90.33 & 99.10 & 79.20 & 15.69 & 89.26 & 99.09 & 80.80 \\*
\midrule
\textcolor{blue!45!black}{ClockMark} & 13.22 & \textbf{99.57} & \textbf{100.00} & \textbf{99.80} & 12.78 & \textbf{99.52} & \textbf{100.00} & \textbf{100.00} \\
\midrule
\addlinespace[4pt]
\multicolumn{1}{c}{\multirow{2}{*}{Method}} & \multicolumn{4}{c}{\textbf{LFQA (64 bits)}} & \multicolumn{4}{c}{\textbf{C4 (64 bits)}} \\*
\cmidrule(lr){2-5}\cmidrule(lr){6-9}
 & PPL $\downarrow$ & BA $\uparrow$ & AUC $\uparrow$ & TPR $\uparrow$ & PPL $\downarrow$ & BA $\uparrow$ & AUC $\uparrow$ & TPR $\uparrow$ \\*
\midrule
No watermark & 13.42 & -- & -- & -- & 12.47 & -- & -- & -- \\*
\midrule
BalanceMarking & 14.36 & 57.05 & -- & -- & 13.68 & 56.02 & -- & -- \\*
CredID & 15.06 & 59.52 & -- & -- & 14.03 & 58.39 & -- & -- \\*
PURA & 13.46 & 62.91 & -- & -- & 12.69 & 61.81 & -- & -- \\*
RSBH & 14.38 & 52.28 & -- & -- & 13.78 & 52.12 & -- & -- \\*
WaterMod & 21.72 & 87.19 & -- & -- & 18.47 & 84.49 & -- & -- \\*
\midrule
BiMark & 13.20 & 83.89 & \textbf{100.00} & \textbf{99.80} & 12.64 & 83.38 & \textbf{100.00} & \textbf{100.00} \\*
MOSAIC & 15.80 & 51.46 & 50.05 & 0.20 & 14.79 & 51.79 & 49.99 & 0.00 \\*
MPAC & 17.50 & 81.85 & 99.22 & 95.00 & 15.99 & 80.00 & 99.45 & 95.00 \\*
QuantileMark & 13.31 & 50.17 & 49.60 & 1.20 & 12.73 & 49.29 & 49.61 & 0.80 \\*
StealthInk & 13.62 & 80.25 & 98.17 & 83.40 & 12.98 & 79.50 & 98.24 & 78.80 \\*
WeaveMark & 13.60 & 92.57 & 99.97 & 99.40 & 12.63 & 89.46 & 99.98 & 99.40 \\*
XMark & 17.12 & 80.38 & 97.06 & 32.80 & 15.23 & 78.22 & 96.70 & 47.00 \\*
\midrule
\textcolor{blue!45!black}{ClockMark} & 13.20 & \textbf{97.33} & \textbf{100.00} & \textbf{99.80} & 12.90 & \textbf{97.02} & \textbf{100.00} & \textbf{100.00} \\
\bottomrule
\end{longtable}
\endgroup

\FloatBarrier
\subsection{Robustness on LFQA}
\label{app:robust-lfqa}
Table~\ref{tab:robust-lfqa} reports LFQA results under the same editing attacks as the C4 comparison in Section~\ref{sec:robustness}. Deletion and substitution use rates of 10\%, 30\%, and 50\%. Paraphrasing uses DIPPER-XXL with lexical and order diversity both set to 20. ClockMark achieves the highest BA and TPR across all seven conditions.
\begin{table}[t]
\papertablestyle
\renewcommand{\arraystretch}{1.0}
\setlength{\tabcolsep}{2pt}
\caption{LFQA robustness for 32 bit Qwen watermarks. Bit accuracy (BA), area under the ROC curve (AUC), and true positive rate (TPR) at a 1\% false positive rate (FPR) are percentages. Bold marks the best displayed BA, AUC, and TPR. Baselines without reported detection appear first. Dashes indicate unavailable or inapplicable results.}\label{tab:robust-lfqa}
\begin{tabular*}{\textwidth}{@{\extracolsep{\fill}}lrrrrrrrrr@{}}
\toprule
\multicolumn{1}{c}{\multirow{3}{*}[-4pt]{Method}} & \multicolumn{9}{c}{\textbf{Deletion}} \\
\cmidrule(lr){2-10}
 & \multicolumn{3}{c}{10\%} & \multicolumn{3}{c}{30\%} & \multicolumn{3}{c}{50\%} \\
\cmidrule(lr){2-4}\cmidrule(lr){5-7}\cmidrule(lr){8-10}
 & BA & AUC & TPR & BA & AUC & TPR & BA & AUC & TPR \\
\midrule
BalanceMarking & 53.77 & -- & -- & 37.35 & -- & -- & 25.26 & -- & -- \\
CredID & 53.81 & -- & -- & 37.83 & -- & -- & 24.41 & -- & -- \\
PURA & 72.51 & -- & -- & 58.13 & -- & -- & 53.54 & -- & -- \\
RSBH & 51.38 & -- & -- & 50.97 & -- & -- & 50.28 & -- & -- \\
WaterMod & 57.56 & -- & -- & 53.62 & -- & -- & 51.26 & -- & -- \\
\midrule
BiMark & 76.58 & 97.31 & 79.33 & 60.95 & 74.73 & 14.00 & 47.78 & 54.11 & 2.00 \\
MOSAIC & 51.48 & 49.88 & 0.22 & 50.84 & 50.10 & 0.67 & 50.36 & 49.88 & 0.22 \\
MPAC & 76.93 & 92.39 & 44.22 & 68.31 & 81.81 & 17.33 & 60.75 & 70.20 & 5.33 \\
QuantileMark & 64.83 & 64.82 & 7.56 & 52.69 & 50.35 & 1.11 & 50.10 & 50.34 & 0.89 \\
StealthInk & 69.30 & 73.09 & 9.33 & 57.49 & 51.16 & 0.89 & 52.83 & 43.51 & 0.44 \\
WeaveMark & 59.52 & 96.02 & 60.00 & 51.91 & 82.25 & 18.00 & 50.58 & 66.10 & 4.44 \\
XMark & 72.87 & 69.06 & 5.33 & 61.08 & 12.33 & 0.22 & 53.66 & 0.79 & 0.22 \\
\midrule
\textcolor{blue!45!black}{ClockMark} & \textbf{89.85} & \textbf{99.51} & \textbf{94.89} & \textbf{81.03} & \textbf{94.45} & \textbf{56.89} & \textbf{70.27} & \textbf{76.56} & \textbf{15.33} \\
\bottomrule
\end{tabular*}
\par\vspace{6pt}
\begin{tabular*}{\textwidth}{@{\extracolsep{\fill}}lrrrrrrrrrrrr@{}}
\toprule
\multicolumn{1}{c}{\multirow{3}{*}[-4pt]{Method}} & \multicolumn{9}{c}{\textbf{Substitution}} & \multicolumn{3}{c}{\multirow{2}{*}[-2pt]{\textbf{Paraphrasing}}} \\
\cmidrule(lr){2-10}
 & \multicolumn{3}{c}{10\%} & \multicolumn{3}{c}{30\%} & \multicolumn{3}{c}{50\%} & \multicolumn{3}{c}{} \\
\cmidrule(lr){2-4}\cmidrule(lr){5-7}\cmidrule(lr){8-10}\cmidrule(lr){11-13}
 & BA & AUC & TPR & BA & AUC & TPR & BA & AUC & TPR & BA & AUC & TPR \\
\midrule
BalanceMarking & 53.25 & -- & -- & 51.28 & -- & -- & 49.92 & -- & -- & 50.03 & -- & -- \\
CredID & 53.12 & -- & -- & 50.33 & -- & -- & 50.60 & -- & -- & 51.13 & -- & -- \\
PURA & 66.34 & -- & -- & 54.88 & -- & -- & 51.62 & -- & -- & 62.49 & -- & -- \\
RSBH & 50.94 & -- & -- & 49.67 & -- & -- & 49.75 & -- & -- & 51.07 & -- & -- \\
WaterMod & 54.99 & -- & -- & 52.67 & -- & -- & 51.62 & -- & -- & 52.91 & -- & -- \\
\midrule
BiMark & 70.96 & 91.69 & 49.33 & 54.69 & 61.89 & 4.00 & 47.79 & 52.49 & 1.11 & 62.56 & 77.80 & 16.89 \\
MOSAIC & 51.00 & 49.99 & 0.44 & 51.42 & 49.99 & 0.44 & 50.92 & 49.88 & 0.22 & 50.49 & 49.88 & 0.22 \\
MPAC & 71.78 & 84.15 & 25.56 & 61.91 & 62.89 & 4.22 & 56.16 & 51.70 & 1.11 & 65.47 & 70.29 & 7.78 \\
QuantileMark & 59.56 & 57.75 & 3.33 & 52.50 & 50.11 & 1.56 & 49.55 & 49.98 & 1.56 & 57.03 & 54.76 & 2.44 \\
StealthInk & 64.08 & 63.48 & 6.67 & 54.19 & 50.83 & 1.11 & 51.28 & 48.53 & 1.11 & 60.12 & 58.18 & 2.67 \\
WeaveMark & 55.19 & 92.71 & 41.33 & 49.88 & 72.05 & 7.11 & 49.92 & \textbf{58.02} & 2.00 & 51.97 & 83.75 & 19.56 \\
XMark & 66.79 & 71.03 & 7.33 & 56.72 & 50.89 & 2.44 & 52.04 & 49.69 & 2.67 & 61.07 & 53.92 & 3.78 \\
\midrule
\textcolor{blue!45!black}{ClockMark} & \textbf{84.94} & \textbf{97.47} & \textbf{78.89} & \textbf{70.58} & \textbf{76.65} & \textbf{14.44} & \textbf{61.11} & 57.77 & \textbf{3.33} & \textbf{78.02} & \textbf{90.75} & \textbf{41.56} \\
\bottomrule
\end{tabular*}
\end{table}

\paragraph{Recovery and Detection Withstand Token Deletion}
Across deletion rates, ClockMark retains the highest BA, ranging from 89.85\% at 10\% deletion to 70.27\% at 50\%. The largest TPR lead occurs at 30\%, with 56.89\% against 18.00\% for WeaveMark. At 50\%, TPR remains 15.33\% versus 5.33\% for MPAC. The recovery and detection advantage therefore persists as more tokens are removed.

\paragraph{Recovery Remains Stronger Under Substitution}
Substitution causes larger losses than deletion. At 30\% substitution, ClockMark retains 70.58\% BA and 14.44\% TPR, exceeding MPAC at 61.91\% BA and WeaveMark at 7.11\% TPR. At 50\%, ClockMark leads in BA (61.11\%) and TPR (3.33\%), while WeaveMark has a slightly higher AUC (58.02\% versus 57.77\%). The low TPR across methods shows that heavy substitution remains difficult to detect at 1\% FPR.

\paragraph{Robustness Extends to Paraphrasing}
Under paraphrasing, ClockMark achieves 78.02\% BA against 65.47\% for MPAC, a gain of 12.55 percentage points. AUC and TPR reach 90.75\% and 41.56\%, exceeding WeaveMark, the strongest detection baseline, at 83.75\% and 19.56\%. These gains extend the robustness advantage to paraphrasing on LFQA.

\end{document}